\documentclass[10pt]{article}

\usepackage{comment,url,algorithm,algorithmic,graphicx,subcaption,relsize}
\usepackage{amssymb,amsfonts,amsmath,amsthm,amscd,dsfont,mathrsfs,mathtools,microtype,nicefrac,pifont}
\usepackage{float,psfrag,epsfig,color,url,hyperref}
\usepackage{upgreek}
\usepackage[dvipsnames]{xcolor}
\usepackage{epstopdf,bbm,mathtools,enumitem}
\usepackage[toc,page]{appendix}
\usepackage[mathscr]{euscript}
\usepackage[normalem]{ulem}

\usepackage[top=1in, bottom=1in, left=1in, right=1in]{geometry}

\def\balign#1\ealign{\begin{align}#1\end{align}}
\def\baligns#1\ealigns{\begin{align*}#1\end{align*}}
\def\balignat#1\ealign{\begin{alignat}#1\end{alignat}}
\def\balignats#1\ealigns{\begin{alignat*}#1\end{alignat*}}
\def\bitemize#1\eitemize{\begin{itemize}#1\end{itemize}}
\def\benumerate#1\eenumerate{\begin{enumerate}#1\end{enumerate}}

\newenvironment{talign*}
 {\csname align*\endcsname}
 {\endalign}
\newenvironment{talign}
 {\csname align\endcsname}
 {\endalign}

\def\balignst#1\ealignst{\begin{talign*}#1\end{talign*}}
\def\balignt#1\ealignt{\begin{talign}#1\end{talign}}
\let\originalleft\left
\let\originalright\right
\renewcommand{\left}{\mathopen{}\mathclose\bgroup\originalleft}
\renewcommand{\right}{\aftergroup\egroup\originalright}

\def\tinycitep*#1{{\tiny\citep*{#1}}}
\def\tinycitealt*#1{{\tiny\citealt*{#1}}}
\def\tinycite*#1{{\tiny\cite*{#1}}}
\def\smallcitep*#1{{\scriptsize\citep*{#1}}}
\def\smallcitealt*#1{{\scriptsize\citealt*{#1}}}
\def\smallcite*#1{{\scriptsize\cite*{#1}}}

\def\<{\left\langle} % Angle brackets
\def\>{\right\rangle}

\DeclareSymbolFont{rsfs}{U}{rsfs}{m}{n}
\DeclareSymbolFontAlphabet{\mathscrsfs}{rsfs}
\ifdefined\nonewproofenvironments\else
\ifdefined\ispres\else

\hypersetup{
  colorlinks,
  linkcolor={red!50!black},
  citecolor={blue!50!black},
  urlcolor={blue!80!black}
}

\newcommand{\cmark}{\ding{51}}

\newcommand{\xmark}{\ding{55}}

\usepackage{eqparbox}
\usepackage[numbers]{natbib}

\usepackage{threeparttable}
\usepackage{multirow}
\usepackage{makecell}

\DeclareMathOperator*{\argmin}{arg\,min}

\newtheorem{mydefinition}{Definition}
\newtheorem{claim}{Claim}
\newtheorem{assumption}{Assumption}
\newtheorem{theorem}{Theorem}[section]
\newtheorem{proposition}{Proposition}[section]
\newtheorem{lemma}{Lemma}[section]
\newtheorem{corollary}{Corollary}[section]

\mathtoolsset{showonlyrefs}

\date{}
\begin{document}

\title{Generalization Bounds for Stochastic Gradient Descent \\via Localized $\varepsilon$-Covers}

 \author{
 Sejun Park\thanks{
  Department of Artificial Intelligence at
  Korea University, \texttt{sejun.park000@gmail.com}
 }
 \ \ \ \ \ \ 
  Umut \c{S}im\c{s}ekli\thanks{DI ENS, Ecole Normale Supérieure, Université PSL, CNRS, INRIA \texttt{umut.simsekli@inria.fr}}
 \ \ \ \ \ \
 Murat A. Erdogdu\thanks{
  Department of Computer Science at
  University of Toronto, and Vector Institute, \texttt{erdogdu@cs.toronto.edu}
 }
}

\maketitle

\begin{abstract}%
\noindent In this paper, we propose a new covering technique localized for the trajectories of SGD. This localization provides an algorithm-specific complexity measured by the covering number, which can have dimension-independent cardinality in contrast to standard uniform covering arguments that result in exponential dimension dependency.
Based on this localized construction, we show that if the objective function is a finite perturbation of a piecewise strongly convex and smooth function with $P$ pieces, i.e.\ non-convex and non-smooth in general, the generalization error can be upper bounded by $O(\sqrt{(\log n\log(nP))/n})$, 
where $n$ is the number of data samples.
In particular, this rate is independent of dimension and does not require early stopping and decaying step size.
Finally, we employ these results in various contexts and derive generalization bounds for multi-index linear models, multi-class support vector machines, and $K$-means clustering for both hard and soft label setups, improving the known state-of-the-art rates.
\end{abstract}

\vspace{-.05in}
\section{Introduction}\label{sec:intro}
\vspace{-.05in}
We consider the following stochastic optimization problem
\begin{align}
    \min_{\theta\in\Theta} \Big\{F(\theta):=\mathbb E_{ Z}[f(\theta;Z)]\Big\},\label{eq:rm}
\end{align}
where $\theta$ represents the optimization parameter,
$\Theta\subset\mathbb R^d$ is a convex parameter domain,
$f(\,\cdot\,;z)$ is a possibly non-convex loss incurred by
a single data point $z \in \mathcal Z$, and $Z$ is a random variable on
$\mathcal Z$ following the data distribution. Since the distribution of $Z$
is unknown in general, the following proxy based on independent and
identically distributed (i.i.d.) samples $z_1,\dots,z_n$ of $Z$
is optimized instead
\begin{align}
    \min_{\theta\in\Theta} \Big\{\hat F(\theta):=\frac1n\sum_{i=1}^nf(\theta;z_i)\Big\}\label{eq:erm}.
\end{align}

\vspace{-.05in}
Given a learning algorithm $A(\,\cdot\,)$ mapping samples $z_1,\dots,z_n$ to an approximate solution of \eqref{eq:erm}, bounding the \emph{generalization error}\footnote{This quantity is also referred to as the generalization gap or the estimation error in the literature.}  $\hat F(A(z_1,\dots,z_n))- F(A(z_1,\dots,z_n))$ is a fundamental problem in learning theory. 
Classical \emph{algorithm-independent} results rely on uniform convergence over the entire domain $\Theta\subset \mathbb{R}^d$;
thus, they apply to any learning algorithm. However, these bounds often increase with the dimension $d$ \cite{shalev09,shalev14,feldman16}, becoming vacuous in the modern overparameterized regime \cite{zhang21b}.
To derive dimension-independent bounds, researchers
have been investigating \emph{algorithm-dependent} generalization properties, especially for commonly used training methods such as stochastic gradient descent (SGD) \cite{hardt15,soudry18,yun21}.

Notably, \emph{algorithmic stability} is a technique 
for deriving generalization bounds based on the properties of
a specific learning algorithm, which leverages that if a parameter learned by an algorithm is robust under a perturbation of samples $z_1,\dots,z_n$, then the generalization error at that parameter must be small~\cite{bousquet02}.
Based on this principle, several works proposed dimension-independent generalization bounds for SGD and its variants under various setups~\cite{hardt15,london17,feldman19,li19,bassily20,lei20,farnia21,lei21b,kozachkov22}.  Bounds derived using algorithmic stability is optimal for strongly convex and smooth functions \cite{shalev09} and convex and non-smooth functions \cite{bassily20,amir21a}.
Nevertheless, without global (strong) convexity, early stopping, and/or decaying step size,
generalization bounds based on algorithmic stability often diverge with the number of SGD iterations~\cite{hardt15,li19}, failing to explain the empirical observations.

To obtain (ambient) dimension-independent bounds that do not diverge with the number of iterations,
recent works proposed to utilize the low-dimensional fractal structures generated by the SGD iterates
whose complexity can be measured by a notion called the Hausdorff dimension \cite{falconer14}. 
In this context, \cite{csimcsekli20} showed that, under a continuous-time surrogate for SGD,
the generalization error can be bounded by $\widetilde O(\sqrt{d_H/n})$,
where $d_H$ denotes the Hausdorff dimension of the optimization trajectory.
This result was later extended to discrete-time iterated function systems by~\cite{camuto21}.
Here, the Hausdorff dimension can be smaller than the ambient dimension \cite{camuto21}, ultimately providing improved generalization bounds. However, both of these results are inherently asymptotic,
and rely on opaque assumptions that are hard to verify in practice.

In this paper, we propose a new framework for deriving generalization bounds for the projected SGD
with a constant step size and without requiring early stopping.
Inspired by the works \cite{csimcsekli20,camuto21},
our framework is based on a complexity measure of the trajectory of SGD, 
which can be quantified under standard verifiable conditions.
Our contributions are as follows.
\begin{itemize}[noitemsep,leftmargin=8pt]%
\item 
\textbf{Localized $\varepsilon$-covers for SGD.} 
Our first principle contribution is a covering technique localized for the trajectories of SGD.
This localization provides an algorithm-specific complexity measured by the covering number, which can have dimension-independent cardinality in contrast to standard covering arguments that result in exponential dimension dependency.

\item\textbf{Generalization bounds for SGD.} Based on this localized covering, we establish dimension-independent generalization bounds for SGD, for a rich class of non-convex loss functions $f$ whose gradients can be approximated by that of a piecewise strongly convex and smooth function $h$, i.e. $\|\nabla f(\theta;z)-\nabla h(\theta;z)\|\le\xi$ for some $\xi$. 
In particular, with high probability, we prove the bound
\begin{align}\label{eq:intro-disp}
\left|\hat F(\theta^{(t)})- F(\theta^{(t)})\right|=O\left(\sqrt{\frac{\log n\log(nP)}{n}}+\xi\right),
\end{align}
where $\theta^{(t)}$ denotes the parameter generated by $t$ SGD iterations for a sufficiently large $t$, and $P$ denotes the number of strongly convex pieces needed to approximate $f$. We further show that the gradient of any (piecewise) smooth function $f$ can be approximated with that of a piecewise strongly
convex and smooth function, demonstrating the wide applicability of the bound \eqref{eq:intro-disp}.
Finally in the special case $P=1$ and $\xi=0$,
our result reduces to a non-asymptotic bound where the complexity is captured by
the Hausdorff dimension of the invariant measure of SGD.

\item \textbf{Improved bounds in specific models.}
  We employ the above result to derive generalization bounds in several statistical models trained by SGD, including multi-index linear models, multi-class support vector machines, and $K$-means clustering with both hard and soft label setups, improving the previously known state-of-the-art generalization error bounds in this context.
 
\end{itemize}

\textbf{Notation and problem setup.}
For $k\in\mathbb N$, we denote $[k]:=\{1,\dots,k\}$
We use $\|\cdot\|$ to denote the $\ell_2$-norm.
For $\varepsilon>0$ and $\theta\in\mathbb R^d$, we use $\mathcal B_{\varepsilon}^d(\theta)$ to denote the $d$-dimensional closed $\ell_2$-ball of radius $\varepsilon$, centered at $\theta$. 
Given a set $\mathcal S\subseteq\mathbb R^d$ and $\varepsilon>0$, we say $\mathcal C_\varepsilon\subseteq\mathbb R^d$ is an ``$\varepsilon$-cover'' of $\mathcal S$ if $\mathcal S\subseteq\bigcup_{\theta\in\mathcal C_\varepsilon}\mathcal B_\varepsilon^d(\theta)$.

Given a step size $\eta>0$, an initial parameter $\theta^{(0)}\in\Theta$, and a randomly sampled index $i_t\in[n]$, 
the $t$-th iteration of the projected SGD performs the following update on the parameters \begin{align}\label{eq:sgd-update}
    \theta^{(t)}=g_{i_t}(\theta^{(t-1)}):=\Pi_{\Theta}\big(\theta^{(t-1)}-\eta\nabla f(\theta^{(t-1)};z_{i_t})\big)\ \ \text{ for }\ t=1,2,...,
\end{align}
where $\Pi_{\Theta}(\theta):=\argmin_{\theta^\prime\in\Theta}\|\theta^\prime-\theta\|$ denotes the Euclidean projection.
The domain $\Theta$ is convex; thus, the projection operation is unique. 
We note that the projection is not needed under the presence of an $\ell_2$-regularizer and Lipschitz continuity; see Section~\ref{sec:discussion} for more details. Throughout the paper, 
we use $\theta^{(t)}:=g_{i_t}\circ\cdots\circ g_{i_1}(\theta^{(0)})$ for possibly random indices $i_1,\dots,i_t$.

Lastly, we recall a few standard notions.
$f:\Theta\rightarrow\mathbb R$ is called ``$\alpha$-strongly convex''  and ``$\beta$-smooth'' respectively
if for all $\theta,\theta^\prime\in\Theta$, the following conditions are satisfied 
\begin{align*}
    f(\theta)-f(\theta^\prime)-\nabla f(\theta^\prime)^\top(\theta-\theta^\prime)\ge\frac\alpha2\|\theta-\theta^\prime\|^2\!,\ \ \ \text{ and }\ \ \ \|\nabla f(\theta)-\nabla f(\theta^\prime)\|\le\beta\|\theta-\theta^\prime\|.
\end{align*}
The function $f$ is called ``convex'' if it is $0$-strongly convex.

\section{Main results}\label{sec:mainresults}
We demonstrate our covering construction in Section~\ref{sec:strconvsmooth}
in the classical strongly convex and smooth case in which
the localization argument can be simplified by the contractivity of SGD.
We present our main generalization result on non-convex losses in Section~\ref{sec:piecestrconvsmooth}, and its implications in Section~\ref{sec:mainapplications}. 
\subsection{A localized covering construction: Strongly convex and smooth case}
\label{sec:strconvsmooth}
To motivate our approach, let us first briefly discuss the limitations of prior methods that are based on uniform convergence of empirical processes. 
Given a set of parameters $\Theta \subseteq\mathcal B_R^d(0)$, let $\mathcal C_\varepsilon$ be an  $\varepsilon$-cover of $\Theta$, i.e. $\Theta\subseteq\bigcup_{\phi\in\mathcal C_\varepsilon}\mathcal B_\varepsilon^d(\phi)$.
Uniform convergence over $\Theta$ can be 
established with high probability by simply applying the union bound over $\mathcal C_\varepsilon$,
which yields a generalization error bound depending on the cardinality of the cover $\sqrt{\log|\mathcal C_\varepsilon|}$. 
However, $\mathcal C_\varepsilon$ is typically independent of the algorithm being used, and standard $\varepsilon$-covers for $\Theta$ yield $|\mathcal C_\varepsilon|=(R/\varepsilon)^{\Omega(d)}$; thus, bounds based on covering numbers often grow with $\sqrt{d}$, which can be loose if $d$ is large.
To overcome this issue, we use the contractive properties of SGD in the strongly convex and smooth case and \emph{localize}
the $\varepsilon$-cover. Namely, instead of covering the entire feasible set $\Theta$,
we construct a cover that contains only the points that can be reached by SGD trajectories,
resulting in a covering number that is independent of the ambient dimension $d$.
We introduce the following sets produced by SGD trajectories.

\begin{mydefinition}\label{def:cover}
  We define the following two subsets of $\Theta$.
  \vspace{-.1in}
  \begin{itemize}[leftmargin=20pt,noitemsep]
  \item The set of points that can be reached by $T$ SGD iterations initialized at $\theta^{(0)}\in \Theta$,
    \begin{align*}
    \Psi_T(\theta^{(0)}) :=
    \{g_{i_T}\circ\cdots\circ g_{i_1}(\theta^{(0)}): i_1,\dots,i_T\in[n]\}.
    \end{align*}
  \item The set of points that can be reached by any $t\geq T$ SGD iterations initialized at $\theta^{(0)}\in \Theta$,
    \begin{align*}
    \Psi_{\geq T}(\theta^{(0)}) := \bigcup_{t\geq T}\Psi_t(\theta^{(0)}).
    \end{align*}
  \end{itemize}
\end{mydefinition}

For $\gamma\in(0,1)$, a function $g:\Theta\!\rightarrow\!\Theta$ is called ``$\gamma$-contractive'' if for all $\theta,\theta^\prime\in\Theta$, it satisfies $\|g(\theta)-g(\theta^\prime)\|\le\gamma{\|\theta-\theta^\prime\|}$. In the constant step-size case, SGD iterates converge to a distribution instead of a single point~\cite{dieuleveut20}, but their contractivity can still provide the following localization: {all possible SGD iterates after sufficiently many iterations can be $\varepsilon$-covered by $n^{O(\log(1/\varepsilon))}$ points.}

\begin{lemma}\label{lem:cover}
Suppose that $g_1,g_2,\dots$ are $\gamma$-contractive for some $\gamma \in (0,1)$.
 Then, for any initialization $\theta^{(0)}\in\Theta\subseteq\mathcal B_R^d(0) $ and for any {$\varepsilon >0$, for
 $T \coloneqq T_\varepsilon=\max\left\{\left\lceil\frac{\log(R/\varepsilon)}{\log(1/\gamma)}\right\rceil,0\right\}$}, 
 we have
  \begin{align*}
  \Psi_{\geq T}(\theta^{(0)}) \subseteq \bigcup_{\phi \in \Psi_T(0)}\mathcal B_\varepsilon^d(\phi).
  \end{align*}
\end{lemma}
\begin{proof}
For $t\ge T$,   let $\theta^{(t)} \in   \Psi_{\geq T}(\theta^{(0)})$ such that $\theta^{(t)}:=g_{i_t}\circ\cdots\circ g_{i_1}(\theta^{(0)})$ for some $i_1,\dots,i_t\in[n]$.
  Let $\phi:=g_{i_t}\circ\cdots\circ g_{i_{t-T+1}}(0)$ and notice that $\phi \in \Psi_T(0)$ by construction. Then for $T \coloneqq T_\varepsilon$ and $\theta^{(t-T)}:=g_{i_{t-T}}\circ\cdots\circ g_{i_1}$, we have  
  \begin{align*}
    \|\theta^{(t)}-\phi\|\,{\le\gamma^T\|\theta^{(t-T)}-0\|}\le\gamma^TR\, \le\ \varepsilon,
  \end{align*}
 {since each $g_{i}$ is $\gamma$-contractive and $\|\theta^{(t-T)}\|\le R$ by \eqref{eq:sgd-update}.}
 This implies
  $
  \theta^{(t)} \in \bigcup_{\phi \in \Psi_T(0)}\mathcal B_\varepsilon^d(\phi).
  $ 
\end{proof}
For strongly convex and smooth $f(\,\cdot\,;z)$, an SGD update $g_i$ with
a sufficiently small step size is contractive~\cite{dieuleveut20}; 
that is, applying $g_i$ to any two points decreases the distance between them (see Appendix~\ref{sec:contractivity}
for a formal derivation). 
Therefore, for any $\varepsilon>0$, there exists $T\coloneqq T_\varepsilon$ such that applying $T$ \emph{synchronously coupled} SGD updates that use the same sample at each iteration can make the distance between the initial points smaller than $\varepsilon$; see Figure~\ref{fig:sample-dependent} (left). 
This observation implies that the set $\Psi_T(0)$ of all parameters that can be generated by $T$ SGD updates when initialized at the origin $\varepsilon$-\emph{covers} the set $\Psi_{\geq T}(\theta^{(0)})$ of all parameters that can be generated by any $t\ge T$ SGD updates for an arbitrary initialization $\theta^{(0)}$; see Figure~\ref{fig:sample-dependent} (right). 
In contrast to algorithm-independent covers of $\Theta$ that scale with $|\mathcal C_\varepsilon|=(R/\varepsilon)^{\Omega(d)}=e^{\widetilde\Omega(d)}$,
we obtain $|\Phi_T(0)|\le n^{T}=e^{\widetilde O(1)}$ which is independent of the dimension $d$ and only polynomial in the number of samples $n$.

\begin{figure}
\centering
\includegraphics[width=\textwidth]{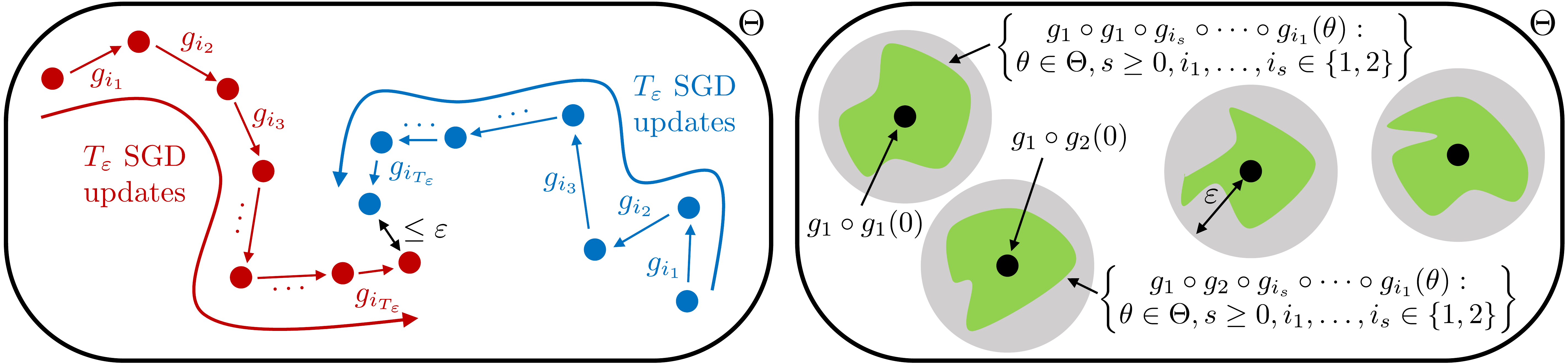}
\caption{Left: illustration of $T=T_\varepsilon$ coupled projected SGD updates from two distinct points. Right: illustration of our localized cover $\Psi_T$ covering $\Psi_{\ge T}$ with $T=2$ and $n=2$ and $|\Psi_T|=n^T=4$.}
\label{fig:sample-dependent}
\end{figure}

We make the following additional assumptions on the loss function $f$.
\begin{assumption}[Weak Lipschitz continuity]\label{asm:lipschitz}
  For $L>0$, there exists $h:\Theta\rightarrow\mathbb R$ such that for all $\theta,\theta^\prime\in\Theta$ and $z\in\mathcal Z$, $
      |f(\theta;z)-h(\theta)-(f(\theta^\prime;z)-h(\theta^\prime))|\le L\|\theta-\theta^\prime\|.
  $
\end{assumption}
Notice that $h=0$ reduces to the classical Lipshitz continuity, but the above condition is more general.% in the first argument. 
\begin{assumption}[Bounded deviation]\label{asm:boundedval}
  For $B\!>\!0$, for all  $\theta\!\in\!\Theta$, $
      \big|\sup_{z\in\mathcal Z}f(\theta;z)-\inf_{z\in\mathcal Z}f(\theta;z)\big|\!\le\! B.
  $
\end{assumption}
We note that Assumptions \ref{asm:lipschitz} \& \ref{asm:boundedval} (or their variants) both appear in several algorithmic stability-based results; see e.g. \cite[Thm~3.10]{hardt15}. 
We also highlight that these conditions are invariant to adding regularizers (e.g. consider $f(\theta;z)\leftarrow f(\theta;z)+\frac\lambda2\|\theta\|^2$) which will be useful in Section~\ref{sec:mainapplications}.

\begin{theorem}\label{thm:strconvsmooth}
  Suppose that Assumptions \ref{asm:lipschitz} \& \ref{asm:boundedval} hold and there exist $\alpha,\beta>0$ such that
  $f(\,\cdot\,;z)$ is $\alpha$-strongly convex and $\beta$-smooth on $\Theta\subseteq\mathcal B_R^d(0)$ for all $z\in\mathcal Z$.
  For any $\eta\in(0,2/\beta)$, let $\gamma:=\sqrt{1-2\alpha\eta+\alpha\beta\eta^2}$,
  {$T:=\max\left\{\left\lceil\frac{\log(2LRn)}{\log(1/\gamma)}\right\rceil,0\right\}$.} 
  Then, with probability at least $1-\delta$, for any $\theta^{(0)}\in\Theta$, $t\ge T$, and $i_1,\dots,i_t\in[n]$, SGD iterate $\theta^{(t)}$ satisfies
  \begin{align}\label{eq:strsmth-1}
    \left|\hat F(\theta^{(t)})-F(\theta^{(t)})\right|\le\frac{BT+1}{n}+B\sqrt{\frac{T\log n+\log(2/\delta)}{2n}}.
  \end{align}
\end{theorem}
\textbf{Remark.} A few remarks are in order. 
For stochastic convex optimization, the best achievable bound via uniform convergence over $\Theta$ is worse than the algorithmic stability-based bounds for SGD by a factor of $\sqrt{d}$~\cite{shalev09,feldman16,feldman19}; however, by localizing the uniform convergence argument, we are able to obtain a dimension-independent bound.
We note that this high-probability result is still not directly comparable to most stablity-based bounds which are given in expectation. An exception is \cite[Thm~4.5]{feldman19} which obtains near-optimal bounds for SGD in the convex case by tuning the number of iterations to be taken, obtaining a rate of $O(\log(n/\delta)^2\log (n) /\sqrt{n})$. 
In contrast, our bound reads $O(\sqrt{\log(n)^2\log(1/\delta)/n})$ which is better by a logarithmic factor and holds for any sufficiently large number of SGD iterations.
We emphasize that the bound \eqref{eq:strsmth-1} holds for \emph{the union of all trajectories of SGD} generated by $t\ge T$ iterations, whereas the stability-based results often consider a single parameter generated by SGD. A similar setup is considered in Corollary~\ref{cor:strconvsmooth} where we establish improved bounds on the generalization error removing the logarithmic factor. 
{Lastly, we note that our localized covering can also be used for deriving bounds in expectation; see Appendix~\ref{sec:expbound}.}

The proof of Theorem~\ref{thm:strconvsmooth} follows from three steps. $i)$ Lemma~\ref{lem:cover} implies that any $\theta^{(t)}$ for $t\ge T$ can be approximated by some parameter $\theta^{(T)}$ generated by $T$ SGD updates initialized at 0. $ii)$ $\hat F(\theta^{(T)})$ concentrates around $F(\theta^{(T)})$ since $\theta^{(T)}$
is \emph{almost} independent of the samples, i.e. it depends on at most $T=O(\log(n))$ of $n$ samples. $iii)$ The empirical process $|\hat F(\theta^{(T)})-F(\theta^{(T)})|$ uniformly converges over the set $\Psi_T(0)$ which has a dimension-free cardinality. See Appendix~\ref{sec:pfthm:strconvsmooth} for details.

While the bound in Theorem~\ref{thm:strconvsmooth} holds uniformly over all possible initializations, it can be tightened by considering a single realization of $\theta^{(0)}$ in the following corollary. 
\begin{corollary}\label{cor:strconvsmooth}
  Assume the setup in Theorem~\ref{thm:strconvsmooth}. 
  Then, for any $\theta^{(0)}\in\Theta$, $t\ge T$, and $i_1,\dots,i_t\in[n]$, with probability at least $1-\delta$, we have 
  \begin{align*}
    \left|\hat F(\theta^{(t)})-F(\theta^{(t)})\right|\le\frac{BT+1}{n}+B\sqrt{\frac{\log(2/\delta)}{2n}}.
  \end{align*}
\end{corollary}

For sufficiently large $n$, the bound in Corollary~\ref{thm:strconvsmooth} coincides with the tight concentration bound at a fixed $\theta\in\Theta$; however, it has the limitation that at least $t \geq T$ SGD updates are required. 
This is remedied in the next result which does not require strong convexity and smoothness. 
\begin{corollary}\label{thm:early}
  Suppose that Assumption~\ref{asm:boundedval} holds. Then, for any $\theta^{(0)}\in\Theta$, $t\ge0$, and $i_1,\dots,i_t\in[n]$, with probability at least $1-\delta$, we have
    \begin{align*}
    \left|\hat F(\theta^{(t)})-F(\theta^{(t)})\right|\le\frac{Bt}{n}+B\sqrt{\frac{\log(2/\delta)}{2n}}.
  \end{align*}
\end{corollary}
The bound in Corollary~\ref{thm:early} requires early stopping since it diverges as $t\rightarrow\infty$.
However, since Corollary~\ref{thm:early} only requires Assumption~\ref{asm:boundedval}, it can be easily combined with other results, e.g. Corollary~\ref{cor:strconvsmooth}, 
to provide a generalization bound that holds for any number of SGD iterations.
The proofs of Corollaries \ref{cor:strconvsmooth} and \ref{thm:early} are presented in Appendices~\ref{sec:pfcor:strconvsmooth} and \ref{sec:pfthm:early}, respectively.

\textbf{Relation to fractal dimension.} There is an interesting connection to be made between Theorem~\ref{thm:strconvsmooth} and the Hausdorff dimension of the support $\mu$ of the stationary distribution of (projected) SGD. For example, suppose that $\eta\in(0,1)$, $\Theta=\mathcal B_R^d(0)$, $f(\theta;z)=\frac12\|\theta-\theta^*_z\|^2$ for some $\theta^*_z\in\mathcal B_R^d(0)$, %(i.e. $\alpha=\beta=1$ and $\gamma=1-\eta$),
and $\|\theta^*_{z_i}-\theta^*_{z_j}\|\ge2\gamma R$ for all $i\ne j$.
Here, the last assumption can be satisfied with high probability if $d,\eta$ are large enough and $\theta^*_z\sim N(0,\sigma^2I)$, i.e. zero-mean Gaussian with covariance $\sigma^2I$;
see e.g. \cite[Appendix~C]{park21}. 
Under these assumptions, 
the bound in Theorem \ref{thm:strconvsmooth} can be reformulated as
\begin{align}
  \left|\hat F(\theta^{(t)})-F(\theta^{(t)})\right|\le\frac{BT+1}{n}+B\sqrt{\frac{\big\lceil d_H+\frac{\log 2LR}{\log(1/\gamma)}\big\rceil\log n+\log(2/\delta)}{2n}},\label{eq:fractalbound}
\end{align}
where $d_H$ denotes the Hausdorff dimension of $\mu$.
The precise definition of $d_H$ and the derivation of \eqref{eq:fractalbound} are presented in Appendix~\ref{sec:pfeq:fractalbound}.
In \eqref{eq:fractalbound}, $d_H$ replaces the ambient dimension $d$ that appear in algorithm-independent bounds~\citep{shalev14}.
We note that such a connection between the fractal dimension and generalization bounds has been also studied in \citep{csimcsekli20,camuto21}; however, their bounds are asymptotic and require  non-trivial assumptions that are not easy to verify in practice.

{\textbf{Generalization due to contractivity.} 
Contractivity of SGD has also been used to derive generalization bounds in the concurrent work~\citep{kozachkov22}. Although the localized covering construction in Lemma~\ref{lem:cover} relies on the same principle,
our results differ from those in \citep{kozachkov22} on two key aspects.
First and foremost, as we shall see in Section~\ref{sec:piecestrconvsmooth}, the localized covering argument used in Lemma~\ref{lem:cover} can also be applied to the non-convex case in which the uniform stability-based argument used in \cite{kozachkov22} provably breaks down (Appendix~\ref{sec:stabilityfail}); thus, extending the results of \cite{kozachkov22} to cover non-convex objectives is highly non-trivial. 
Further, the generalization bounds in \citep{kozachkov22} are provided in expectation, i.e. $|\mathbb E[\hat F(\theta^{(t)})-F(\theta^{(t)})]|\le O(1/n)$, whereas we provide high-probability bounds, i.e. $|\hat F(\theta^{(t)})-F(\theta^{(t)})| \le \widetilde O(1/\sqrt{n})$. When translated to bounds in expectation, our results read $|\mathbb E[\hat F(\theta^{(t)})-F(\theta^{(t)})]| \le \widetilde O(1/n)$ and $\mathbb E[|\hat F(\theta^{(t)})-F(\theta^{(t)})|] \le\widetilde O(1/\sqrt{n})$; see Appendices~\ref{sec:expbound}~\&~\ref{sec:comparison}. 
}

\subsection{Non-convex case: Perturbations of piecewise strongly convex and smooth functions}\label{sec:piecestrconvsmooth}
Algorithmic stability technique yields (near) optimal rates for strongly convex objectives;
however, when applied to non-convex functions, 
the resulting bounds often diverge with the number of SGD iterations \citep{hardt15,li19}.
The localized covering construction introduced in the previous section remedies this issue, providing more stable generalization bounds. Specifically, we establish a dimension-independent generalization bound for functions that are finite perturbations of piecewise strongly convex and smooth functions. 
We further prove an approximation result and show that any smooth non-convex function can be approximated with a piecewise strongly convex function. Since piecewise strongly convex and smooth functions may not be differentiable on the entire $\Theta$, we will use the auxiliary gradient in Definition~\ref{def:piecestrconvsmooth} as a surrogate for the SGD update \eqref{eq:sgd-update}.

\begin{mydefinition}\label{def:piecestrconvsmooth}
  $f$ is ``piecewise $\alpha$-strongly convex and $\beta$-smooth with $P$ pieces on $\Theta$'' if there exists a partition $\{\mathcal P_{1},\dots,\mathcal P_{P}\}$ of $\Theta$ and $\alpha$-strongly convex and $\beta$-smooth functions $h_1,\dots,h_P$ on $\Theta$ such that $f=h_p$ on $\mathcal P_{p}$ for all $p\in[P]$. 
  We also define $\nabla f(\theta):=\nabla h_p(\theta)$ if $\theta\in\mathcal P_p$.
\end{mydefinition}
Piecewise strongly convex and smooth objectives are widely used in machine learning applications. For example, the objective of learning a single layer of a ReLU network with $\ell_2$-regularization is piecewise strongly convex and smooth. Furthermore, the objective of learning an entire ReLU network is also piecewise strongly convex and smooth on small loss regions \citep{milne18}.
Such observations easily extend to more general settings, e.g. an objective function defined as the summation of a piecewise linear loss (e.g. hinge loss) and $\ell_2$-regularization.
However, if we further allow for finite perturbations, 
any piecewise smooth non-convex function can be covered within this framework.

\begin{proposition}\label{prop:smoothapprox}
For any piecewise $\beta^\prime$-smooth $f$ with $Q$ pieces on $\Theta\subset\mathcal B_R^d(0)$, and for any $\xi>0$ and
$0<\alpha\le\beta$,
there exists a piecewise $\alpha$-strongly convex and $\beta$-smooth $h$ with at most $Q(3(\beta+\beta^\prime)R/\xi)^d$ pieces such that $\|\nabla f(\theta)-\nabla h(\theta)\|\le\xi$ for all $\theta\in\Theta$.
\end{proposition}
For this general class of non-convex functions, we derive the following generalization bound.

\begin{theorem}\label{thm:approxpiecestrconvsmooth}
  Suppose that Assumptions~\ref{asm:lipschitz} \& \ref{asm:boundedval} hold and $\Theta\subseteq\mathcal B_R^d(0)$. Suppose further that there exists $h:\Theta \times \mathcal{Z}\to \mathbb{R}$ such that for any $z\in\mathcal Z$, $h(\,\cdot\,;z)$ is piecewise $\alpha$-strongly convex and $\beta$-smooth with $P$ pieces on $\Theta$ satisfying
  \begin{align}
      \|\nabla f(\theta;z)-\nabla h(\theta;z)\|\le\xi,\label{eq:approxpiecestrconvsmooth}
  \end{align}
  for all $\theta\in\Theta$. 
  For any $\eta\in(0,2/\beta)$, let $\gamma:=\sqrt{1-2\alpha\eta+\alpha\beta\eta^2}$.
  Then given $T\in\mathbb N$, with probability at least $1-\delta$, for any $\theta^{(0)}\in\Theta$, $t\ge T$, and $i_1,\dots,i_t\in[n]$, we have
  \begin{align*}
    \left|\hat F(\theta^{(t)})-F(\theta^{(t)})\right|\le&\frac{BT}{n}+B\sqrt{\frac{T\log(nP)+\log(2/\delta)}{2n}}+2L\bigg(\gamma^{T}R+{\color{black}\frac{1-\gamma^T}{1-\gamma}}\eta\xi\bigg)
  \end{align*}
\end{theorem}

\textbf{Remark.} The bound above is stated in full generality and holds for any SGD iterate $t\geq T$ and any choice of $T \geq 1$. However, to obtain a meaningful generalization bound, 
one may choose, for example $T = O(\log(nR)/\log(\gamma^{-1}))$.
In the case that $\gamma, B,L=\Theta(1)$, the bound simplifies to
  \begin{align}\label{eq:gen-app-short}
    \left|\hat F(\theta^{(t)})-F(\theta^{(t)})\right|\le& O\bigg(\sqrt{\frac{\log(nR)\log(nP)+\log(1/\delta)}{n}}+\xi\bigg).
  \end{align}
  Here, the first term in the bound is logarithmic in the number of pieces $P$ and the last term scales linearly with the approximation error $\xi$. Thus, the generalization bound depends on the trade-off between  the complexity of $h$ through $P$, 
  and how well $\nabla h$ approximates $\nabla f$ through $\xi$.
In this regime, in contrast to algorithmic stability-based bounds, the above result does not grow with the number of SGD iterations, i.e. early stopping is not required for generalization in Theorem~\ref{thm:approxpiecestrconvsmooth}. {We finally note that Theorem~\ref{thm:approxpiecestrconvsmooth} holds for any $f$ with an \emph{auxiliary gradient} $\nabla f$ satisfying~\eqref{eq:approxpiecestrconvsmooth}.
}

In light of Proposition~\ref{prop:smoothapprox},
any piecewise smooth function can be approximated by a piecewise strongly convex and smooth function; in the worse case,
the number of pieces $P$ is at most $e^{\widetilde\Omega(d)}$.
Therefore in this pessimistic case, Theorem~\ref{thm:approxpiecestrconvsmooth} recovers the classical algorithm-independent covering bound $\widetilde O(B\sqrt{d/n})$ by choosing $\xi=O(1/(L\sqrt{n}))$ and $T=\Theta(\log(LRn))$. However, any value of $P$ that is sub-exponential in dimension yields improved generalization bounds. 
In particular in the next section, we consider certain  (non-convex) statistical models and carefully design $h$ so that Theorem~\ref{thm:approxpiecestrconvsmooth} improves the existing generalization error bounds.
{In contrast, uniform stability-based bounds for piecewise strongly convex and smooth functions are in general $\Omega(1)$ after sufficiently many SGD iterations; see Appendix~\ref{sec:stabilityfail}. That is, the contractivity-based bounds in \cite{kozachkov22} cannot be directly extended to piecewise contractivity~\eqref{eq:approxpiecestrconvsmooth}.
}

The proof of Theorem~\ref{thm:approxpiecestrconvsmooth} relies on a modified version of the covering construction presented in Section~\ref{sec:strconvsmooth}. First, we define the auxiliary parameter update $g_{i,p}(\theta):=\theta-\nabla h_p(\theta;z_i)$ where $h_p(\,\cdot\,;z_i)$ denotes the strongly convex and smooth function satisfying {\color{black}$\nabla f(\,\cdot\,;z_i)\approx\nabla h_p(\,\cdot\,;z_i)$} on the $p$-th piece. 
We show that $\Psi_T^\prime(0):=\{g_{i_T,p_T}\circ\cdots\circ g_{i_1,p_1}(0):i_1,\dots,i_T\in[n],p_1\dots,p_T\in[P]\}$ $\varepsilon$-covers $\Psi_{\ge T}(\theta^{(0)})$ for any $\theta^{(0)}\in\Theta$ as in Lemma~\ref{lem:cover}. Since $|\Psi_T^\prime|\le(nP)^T$, applying a concentration inequality and the union bound over the localized cover yields Theorem~\ref{thm:approxpiecestrconvsmooth}.
Formal proofs of Proposition~\ref{prop:smoothapprox} and Theorem~\ref{thm:approxpiecestrconvsmooth} are provided in Appendices~\ref{sec:pfprop:smoothapprox} and \ref{sec:pfthm:approxpiecestrconvsmooth}, respectively.

\section{Applications}\label{sec:mainapplications}
In this section, we use our localized covering to prove generalization bounds for multi-index linear models, multi-class support vector machines, and $K$-means clustering for both hard and soft label setups, trained by SGD, improving the state-of-the-art results known for these models.

\subsection{Multi-index linear models}
Given a sample $z=(y,x)\in \mathcal{Y}\times\mathcal B_{R_x}^d(0)$, consider the $\ell_2$-regularized loss in a
\emph{multi-index model} with $K$ indices parameterized by $\theta=(\theta_j)_{j=1}^K\in\mathbb R^{d\times K}$
\begin{align}
f(\theta;z):=\ell(\theta_1^\top x,\dots,\theta_K^\top x;y)+\sum_{j=1}^K\frac\lambda2\|\theta_j\|^2.\label{eq:glm}
\end{align}
Characterizing the generalization properties of multi-index models is an important problem
with many applications including regression, classification, dimension reduction,
and learning a single-layer of a neural network.
In the following theorem, we derive a generalization bound for multi-index models by approximating each $f(\,\cdot\,;z)$ with a piecewise strongly convex and smooth function,
and then applying Theorem~\ref{thm:approxpiecestrconvsmooth}. The proof is given in Appendix~\ref{sec:pfthm:glm}.
\begin{theorem}\label{thm:glm}
Let $\theta=(\theta_j)_{j=1}^K$ and $\Theta=\prod_{j=1}^K\Theta_j$ for some convex $\Theta_j\subseteq\mathcal B_R^d(0)$. Suppose that $f$ satisfies \eqref{eq:glm}, Assumptions~\ref{asm:lipschitz} \& \ref{asm:boundedval} hold, and
$\ell(\,\cdot\,,\dots,\,\cdot\,;y)$ is piecewise $\beta$-smooth with $Q$ pieces on $\Theta$ for all $y\in\mathcal Y$. 
For any $\eta\in(0,2/\lambda)$, let $\gamma:=|1-\eta\lambda|$, {$T:=\max\left\{\left\lceil\frac{\log(3LR n)}{\log(1/\gamma)}\right\rceil,0\right\}$}, 
and $P:=\mathtt{poly}(\beta,\eta,K,L,R,R_x,T,n)$. Then, with probability at least $1-\delta$, for any $\theta^{(0)}\in\Theta$, $t\ge T$, and $i_1,\dots,i_t\in[n]$,
we have
\begin{align*}
    \left|\hat F(\theta^{(t)})-F(\theta^{(t)})\right|\le\frac{BT+1}{n}+B\sqrt{\frac{T\log(nP^KQ)+\log(2/\delta)}{2n}}.
\end{align*}
\end{theorem}

Generalization behavior of multi-index models has received considerable attention;
a subset of {notable results include~\citep{guermeur02,zhang04,jenssen12,cortes13,lei19}}.
Under the same conditions of Theorem~\ref{thm:glm},
existing state-of-the-art bounds scale at least linearly in $\sqrt{d}$ or $K$,
while our result is dimension-free and scales with $\sqrt{K}$.
Specifically, the result in \cite{lei19} translated to our setting reads
($\ell$ is $L$-Lipschitz, and $f(\theta;z)\in[0,B]$ which are stronger conditions than Assumptions~\ref{asm:lipschitz} \& \ref{asm:boundedval})
\begin{align}
    \left|\hat F(\theta)-F(\theta)\right|
    &\le\widetilde O\left(\frac{{KL RR_x}+B}{\sqrt{n}}\right),\label{eq:glmknown}
\end{align}
for all $\theta\in\Theta$, 
which is obtained via \cite[Cor~3 \& 9]{lei19} with $\Lambda\leftarrow R$ and $p\leftarrow\infty$.
While the bound in \eqref{eq:glmknown} is linear in $KLRR_x$, our bound is linear in $\sqrt{K}$ and logarithmic in $L,R,R_x$. 
The significance of this improvement can be better seen, for example,
when learning the first layer of neural networks, for which $L$ can be very large. Moreover, if $x\sim N(0,I)$, then $\|x\|=\widetilde \Omega(\sqrt{d})$ with high probability; that is, \eqref{eq:glmknown} is linear in $R_x=\Omega(\sqrt{d})$ but ours only scales with $\sqrt{\log d}$,
which can make a significant difference especially in the overparameterized regime. 
We should note that compared to our bound, \eqref{eq:glmknown} does not require the smoothness of $\ell$.

For a specific application of Theorem~\ref{thm:glm},
consider the multi-class support vector machines, i.e. for $\mathcal Y=[K]$ and $\rho:[-RR_x,RR_x]\rightarrow[0,B]$ is $L$-Lipschitz and $\beta$-smooth,
the objective is given as
\begin{align*}
    f(\theta;z)=\max_{y^\prime\ne y}\rho\Big(\theta_y^\top x-\theta_{y^\prime}^\top x\Big)+\frac\lambda2\sum_{j=1}^K\|\theta_j\|^2.
\end{align*}
Theorem~\ref{thm:glm} provides a generalization bound of $\widetilde O(\sqrt{K/n})$
that improves the existing bounds which scale at least linearly with $K$ and/or $\sqrt{d}$ \citep{zhang04,daniely15,lei15,lei19}; see e.g. \cite[Sec~2.1]{lei19} for similar related results.

\subsection{$K$-means clustering}\label{sec:kmeans}
We can also provide generalization error bounds for the problem of $K$-means clustering in both hard and soft label setups.
As before, we consider SGD as an optimizer
which has been the focus of many works in this context~\cite{bottou94,sculley10,tang17}.
We assume throughout this section that the samples are supported on a bounded domain,
and the SGD iterations are applied without projection.

Generalization properties of $K$-means clustering has been studied for decades \citep{antos05,levrard13,thorpe15,tang16}. 
While most existing bounds are at least linear in $K$, \cite{li21}
recently provided an improved bound in the hard label setup, which is of order $(K\log^3 n/n)^{1/2}$ for bounded inputs. 
However, state-of-the-art bounds in the soft label setup were still linear in $K$.
Below, we establish a bound of order $\sqrt{K\log n/n}$ in the soft label setup, which is the first bound that is sublinear in $K$.
We further improve the bounds of \cite{li21} in the hard label setup, but
by a logarithmic factor. 

\textbf{Soft $K$-means clustering.} While the hard $K$-means clustering assigns exactly one cluster to each point using one-hot encoding, the soft $K$-means clustering allows \emph{soft labels} within the probability simplex.
Specifically, given $\zeta>0$ and the samples $z_1,\dots,z_n$, 
the soft $K$-means clustering algorithm alternates between updating the soft labels $(w_j(z_i;\theta))_{j=1}^K$ of $z_i$ for all $i\in[n]$, and estimating the cluster centers $\theta=(\theta_j)_{j=1}^K\in(\mathcal B_R^d(0))^K$ using the classical update rule \citep{mackay03}
\begin{align}
w_j(z_i;\theta)\leftarrow\frac{\exp(-\zeta\|\theta_j-z_i\|^2)}{\sum_{k=1}^K\exp(-\zeta\|\theta_k-z_i\|^2)},\qquad\theta_j\leftarrow\frac{\sum_{i=1}^nw_j(z_i;\theta)z_i}{\sum_{i=1}^nw_j(z_i;\theta)}.\label{eq:softkmeansalg}
\end{align}
Here, this procedure is equivalent to a special case of expectation-maximization algorithm,
which converges to a local minimum of the following objective
(see Appendix~\ref{sec:softkmeansem} 
for details)
\begin{align}
    \hat F(\theta) = \frac1n\sum_{i=1}^n\Big\{f(\theta;z_i)\coloneqq - \frac 1\zeta\log\Big(\sum_{j=1}^K\exp\big(-\zeta\|\theta_j-z_i\|^2\big)\Big)\Big\}.\label{eq:softkmeansem}
\end{align}
Instead of running the standard alternating procedure~\eqref{eq:softkmeansalg},
we directly minimize~\eqref{eq:softkmeansem} using SGD
and derive generalization bounds for the soft $K$-means clustering by approximating the objective \eqref{eq:softkmeansem} with piecewise strongly convex and smooth functions.  
The proof is presented in Appendix~\ref{sec:pfthm:softkmeans}. 
\begin{theorem}\label{thm:softkmeans}
For $B:=4(R+1)^2$ and for any $\eta\in(0,Ke^{-\zeta B})$, 
let $\gamma:=\sqrt{1-\frac{4\eta e^{-\zeta B}}{K}+\frac{4\eta^2}{K^2}}$, $L:=\frac{4R}{\sqrt{K}}e^{\zeta B}$, {$T:=\max\left\{\left\lceil\frac{\log(3LRn)}{\log(1/\gamma)}\right\rceil,0\right\}$}, and $P:=\mathtt{poly}(\eta,\zeta, B,K,R,T,n,e^{\zeta B})$. Then, with probability at least $1-\delta$, for any $\theta^{(0)}\in\Theta$, $t\ge T$, and $i_1,\dots,i_t\in[n]$, we have
\begin{align*}
    \left|\hat F(\theta^{(t)})-F(\theta^{(t)})\right|\le\frac{BT+1}{n}+B\sqrt{\frac{T\log(nP^K)+\log(2/\delta)}{2n}}.
\end{align*}
\end{theorem}
The above generalization bound scales with $\sqrt{K}$, which improves the previously known bound $\widetilde O(K/\sqrt{n})$ for clustering with soft labels \citep{li21}.
Theorem~\ref{thm:softkmeans} also implies $\widetilde O(\sqrt{K/n})$ bound under $R,\zeta=\Theta(1)$ and $\eta=\Theta(K)$, which coincides with the best known rate in the hard label setup.
Here, the choice of $\eta=\Theta(K)$ may look odd; however,
it is indeed practical since the derivative of the objective function $f$ scales with $1/K$, i.e.
$\|{\partial f(\theta;z)}/{\partial \theta_j}\|\le{4Re^{\zeta B}}/{K}$. 

\textbf{Hard $K$-means clustering.}
Under the same setup as before, 
we minimize the objective function
\begin{align}
  \hat F(\theta) = \frac1n\sum_{i=1}^n\Big \{ f(\theta;z_i)\coloneqq\min_{j\in[K]}\|\theta_j-z_i\|^2\Big\}.\label{eq:hardkmeansobj}
\end{align}
Note that \eqref{eq:hardkmeansobj} coincides with the soft $K$-means objective \eqref{eq:softkmeansem} as $\zeta\rightarrow\infty$ if $\|\theta_j-z_i\|\ne\|\theta_k-z_i\|$ for all $i,j,k$.
Since $f(\theta;z)=\min_{j\in[K]}\|\theta_j-z\|^2$ may not be differentiable,
we use an auxiliary gradient at a non-differentiable $\theta$, i.e. for a randomly or deterministically chosen $\mathcal S\subseteq\argmin_{j\in[K]}\|\theta_j-z\|^2$,
\begin{align*}
    \frac{\partial}{\partial\theta_j}f(\theta;z):=\begin{cases}
0~&\text{if}~j\not\in\mathcal S\\
2(\theta_j-z)~&\text{if}~j\in\mathcal S
\end{cases}.
\end{align*}
For example, one may choose a single cluster index ($|\mathcal S|\!=\!1$) and compute the corresponding gradient.
The following result characterizes the generalization of hard $K$-means clustering trained by SGD
 using this auxiliary gradient. 
Its proof is deferred to Appendix~\ref{sec:pfthm:hardkmeans}.
\begin{theorem}\label{thm:hardkmeans} 
  For any $\eta\in(0,1)$, let $B:=4R^2$, $\gamma:=|1-2\eta|$, {$T:=\max\left\{\left\lceil\frac{\log({16}\sqrt{K}R^2n)}{\log(1/\gamma)}\right\rceil,0\right\}$.} 
  Then, given $\theta^{(0)}\in\Theta$, the following bound holds with probability at least $1-\delta$,
  for any $t\ge0$  and $i_1,\dots,i_t\in[n]$, we have
\begin{align*}
    \left|\hat F(\theta^{(t)})-F(\theta^{(t)})\right|\le\frac{BKT+1}{n}+B\sqrt{\frac{KT\log(2n)+\log(2/\delta)}{2n}}.
\end{align*}
\end{theorem}
We note that under $\eta,R\!=\!\Theta(1)$ and $K\!=\!O(n)$, Theorem~\ref{thm:hardkmeans} provides $O((K\log^2n/n)^{1/2})$ generalization bound which improves the bound given by \cite{li21}, but this time by a logarithmic factor. 
\section{Discussions}\label{sec:discussion}
\textbf{Beyond (piecewise) strong convexity and smoothness.}
The localized covering construction we utilized is essentially based on the (piecewise) contractivity of SGD updates for (piecewise) strongly convex and smooth functions. 
However, local strong convexity is by no means necessary and this analysis can be extended to a broader class of objective functions.

Consider, for example, an objective function $f$ that is uniformly convex \citep{doikov21} and has H\"older continuous gradient \citep{nesterov15}. Then, for any $\varepsilon>0$, there exists a step size $\eta>0$ and $T=T_\varepsilon$ such that applying $T$ synchronously coupled SGD updates on any two points makes the distance between them smaller than $\varepsilon$. Consequently, an analog of Lemma~\ref{lem:cover} can be established under this setup as well, for a properly chosen $\eta,\varepsilon,T$. We highlight that this class of functions is already covered in our framework via Proposition~\ref{prop:smoothapprox}; however, an analysis based on the actual curvature (as opposed to a piecewise approximation) of the objective may provide tighter generalization bounds.

\textbf{Extension to SGD without projection.}
Projection operation is only needed to ensure that the SGD iterates stay bounded.
Indeed, the projection is not needed in the presence of explicit regularization, 
or under a dissipativity-type condition on the objective~\citep{raginsky17,erdogdu2018global,yu21, erdogdu2021convergence}.
For example, if $f(\,\cdot\,;z)$ is Lipschitz continuous, it is a straightforward exercise to show that SGD iterations are bounded
in the presence of $\ell_2$-regularization and bounded initialization.

\textbf{Extension to mini-batch setup and different sampling schemes.} 
The results we presented can be easily extended to the mini-batch setting. To see this, note that the localization is based on the (piecewise) contractivity of a single iteration of the algorithm, and we have that if all $g_1,\dots,g_n$ are contractive, then an average of their subset is also contractive.
Moreover, our generalization bounds hold for any sampling scheme, e.g. sampling without replacement, random shuffling, data-dependent sampling, since the covering construction is based on an $\varepsilon$-cover of \emph{the union of all possible trajectories of SGD}, which is independent of the underlying sampling scheme.

\textbf{Extension to contractive stochastic optimization methods.}
In this paper, our main focus was the SGD algorithm and its generalization properties.
However, it is straightforward to adapt our framework to a general stochastic optimization setup. 
For example, the next result follows from the identical steps leading to Theorem~\ref{thm:approxpiecestrconvsmooth}; hence, its proof is omitted.

{
\begin{mydefinition}\label{def:optimizer}
An ``iterative stochastic algorithm'' using $g:\Theta\times\mathcal Z\rightarrow\Theta$ performs the update $\theta^{(t)}=g_{i_t}(\theta^{(t-1)}):=g(\theta^{(t-1)};z_{i_t})$
at iteration $t$, for a random sample $i_t\in[n]$. 
\end{mydefinition}
\begin{mydefinition}\label{def:piececontractive}
$g:\Theta\rightarrow\Theta$ is ``piecewise $\gamma$-contractive with $P$ pieces on $\Theta$'' if there exists a partition $\mathcal P_1,\dots,P_P$ of $\Theta$ and $\gamma$-contractive $h_1,\dots,h_P$ on $\Theta$ such that $g=h_p$ on $\mathcal P_p$ for all $p\in[P]$.
\end{mydefinition}
\begin{theorem}\label{thm:piececontractive}
Suppose that Assumptions~\ref{asm:lipschitz} \& \ref{asm:boundedval} hold, $\Theta\subset\mathcal B_R^d(0)$, and an iterative stochastic algorithm $g$ is given. Suppose further that there exists $h:\Theta\times\mathcal Z\rightarrow\Theta$ such that for any $z\in\mathcal Z$, $h(\,\cdot\,;z)$ is piecewise $\gamma$-contractive with $P$ pieces on $\Theta$ satisfying
\begin{align*}
    \|g(\theta;z)-h(\theta;z)\|\le\xi
\end{align*}
for all $\theta\in\Theta$. Choose $T\in\mathbb N$.
Then with probability at least $1-\delta$, for any $\theta^{(0)}\in\Theta$, $t\ge T$, and $i_1,\dots,i_t\in[n]$,
we have
\begin{align*}
    \left|\hat F(\theta^{(t)})-F(\theta^{(t)})\right|\le&\frac{BT}{n}+B\sqrt{\frac{T\log(nP)+\log(2/\delta)}{2n}}+2L\bigg(\gamma^{T}R+{\frac{1-\gamma^T}{1-\gamma}}\xi\bigg).
\end{align*}
\end{theorem}}
We note that a concurrent work by \cite{kozachkov22} requires \emph{global} contractivity and their bounds are in expectation, whereas Theorem~\ref{thm:piececontractive} is a high-probability statement which only requires \emph{piecewise} contractivity as stated in Definition~\ref{def:piececontractive}.

\textbf{Limitation of our results.} We outline a few limitations of our current analysis. In this paper, we only considered the generalization error (also referred to as the generalization gap), and our main result in the non-convex regime cannot be easily translated to a bound on the excess risk. 
This additional step would require a bound on the optimization error;
nevertheless, establishing such bounds in the non-convex regime is highly non-trivial.
It is worth highlighting that  without assuming convexity, 
algorithmic stability-based bounds also suffer from this limitation \citep{hardt15,feldman19,bousquet20}.

Because of the shared contractivity parameter $\gamma$ across different iterations and pieces,
our method cannot be easily extended to varying step size.
This is indeed a limitation of our analysis as varying step size schedules are oftentimes used in practice.
We leave this extension as a future work.

\section*{Acknowledgments}
SP was supported by Institute of Information \& communications Technology Planning \& Evaluation (IITP) grant funded by the Korea government (MSIT) (No. 2019-0-00079, Artificial Intelligence Graduate School Program, Korea University). US was funded in part by the French
government under management of Agence Nationale de la
Recherche as part of the ``Investissements d’avenir'' program,
reference ANR-19-P3IA-0001 (PRAIRIE 3IA Institute), and the European Research Council Starting Grant DYNASTY -- 101039676. 
MAE was supported by NSERC Grant [2019-06167], CIFAR AI Chairs program, and CIFAR AI Catalyst grant.

\bibliographystyle{amsalpha}
\bibliography{reference}

\newpage
\appendix
\section{Technical results}
All results presented in this section are standard. However, we provide their proofs for convenience.
\begin{lemma}\label{lem:projection}
Let $\mathcal S$ be a Hibert space associated with the norm $\|\cdot\|$ induced by the inner product $\langle\cdot,\cdot\rangle$. Let $\mathcal C\subset\mathcal S$ be a convex set and $\Pi_{\mathcal C}(x):=\arg\min_{z\in\mathcal C}\|x-z\|$ be a projection of $x$ onto $\mathcal C$. Then,
\begin{align*}
    \|\Pi_{\mathcal C}(x)-\Pi_{\mathcal C}(y)\|\le\|x-y\|\quad\forall x,y\in\mathcal S.
\end{align*}
\end{lemma}
\begin{proof}
Since $\mathcal C$ is convex, $\Pi_{\mathcal C}(x)$ is well-defined.
From the definition of the projection and the convexity of $\mathcal C$, we have 
\begin{align}
    \|\Pi_{\mathcal C}(x)-x\|^2\le&\big\|\big((1-t)\Pi_{\mathcal C}(x)+t\Pi_{\mathcal C}(y)\big)-x\big\|^2\notag\\
    =&\big\|\Pi_{\mathcal C}(x)-x+t\big(\Pi_{\mathcal C}(y)-\Pi_{\mathcal C}(x)\big)\big\|^2\label{eq:pflem:projection1}
\end{align}
for all $t\in[0,1]$. Since we have
\begin{align}
    \frac{\partial}{\partial t}\|\Pi_{\mathcal C}(x)-x+t(\Pi_{\mathcal C}(y)-\Pi_{\mathcal C}(x))\|^2=2\langle \Pi_{\mathcal C}(x)-x,\Pi_{\mathcal C}(y)-\Pi_{\mathcal C}(x)\rangle+2t\|\Pi_{\mathcal C}(y)-\Pi_{\mathcal C}(x))\|^2 
\end{align}
and $\frac{\partial}{\partial t}\|\Pi_{\mathcal C}(x)-x+t(\Pi_{\mathcal C}(y)-\Pi_{\mathcal C}(x))\|^2\ge0$ at $t=0$ by \eqref{eq:pflem:projection1}, we have
\begin{align}
    \langle \Pi_{\mathcal C}(x)-x,\Pi_{\mathcal C}(y)-\Pi_{\mathcal C}(x)\rangle\ge0.\label{eq:pflem:projection3}
\end{align}
Likewise, we have 
\begin{align}
    \langle \Pi_{\mathcal C}(y)-y,\Pi_{\mathcal C}(x)-\Pi_{\mathcal C}(y)\rangle\ge0.\label{eq:pflem:projection4}
\end{align}
Now, consider the function 
$$f(t):=\|\Pi_{\mathcal C}(x)-\Pi_{\mathcal C}(y)+t(x-\Pi_{\mathcal C}(x)-y+\Pi_{\mathcal C}(y))\|^2,$$
i.e. $f(0)=\|\Pi_{\mathcal C}(x)-\Pi_{\mathcal C}(y)\|^2$ and $f(1)=\|x-y\|^2$.
Then $f(0)\le f(1)$ since the following inequality holds:
\begin{align*}
    \frac{d}{dt}f(t)&=2\langle \Pi_{\mathcal C}(x)-\Pi_{\mathcal C}(y),x-\Pi_{\mathcal C}(x)-y+\Pi_{\mathcal C}(y)\rangle+2t\|x-\Pi_{\mathcal C}(x)-y+\Pi_{\mathcal C}(y)\|^2\\
    &=2\langle \Pi_{\mathcal C}(x)-x,\Pi_{\mathcal C}(y)-\Pi_{\mathcal C}(x)\rangle+2\langle \Pi_{\mathcal C}(y)-y,\Pi_{\mathcal C}(x)-\Pi_{\mathcal C}(y)\rangle\\
    &\qquad+2t\|x-\Pi_{\mathcal C}(x)-y+\Pi_{\mathcal C}(y)\|^2\\
    &\ge0\quad\forall t\in[0,1]
\end{align*}
where the inequality follows from \eqref{eq:pflem:projection3} and \eqref{eq:pflem:projection4}. This completes the proof of Lemma~\ref{lem:projection}.
\end{proof}
\begin{lemma}[Hoeffding's inequaltiy \citep{hoeffding63}]\label{lem:boundedtail}
Let $X$ be a random variable on $[a,b]$ and $X_1,\dots,X_n$ be independent copies of $X$. Then, it holds that
\begin{align*}
    \mathbb P\left(\left|\frac1n\sum_{i=1}^n X_i-\mathbb E[X]\right|\ge\varepsilon\right)\le2\exp\left(-\frac{2n\varepsilon^2}{(b-a)^2}\right).
\end{align*}
\end{lemma}

\begin{lemma}\label{lem:smoothequiv}
Suppose that $f:\mathbb R^d\rightarrow\mathbb R$ is convex and differentiable. Then, the following conditions are equivalent.
\begin{itemize}
    \item[1.] $f$ is $\beta$-smooth, i.e. $\|\nabla f(\theta)-\nabla f(\theta^\prime)\|\le\beta\|\theta-\theta^\prime\|$ for all $\theta,\theta^\prime\in\mathbb R^d$.
    \item[2.] $\frac{\beta}2\|\theta\|^2-f(\theta)$ is convex for all $\theta\in\mathbb R^d$.
    \item[3.] $f(\theta)-f(\theta^\prime)-\nabla f(\theta^\prime)^\top(\theta-\theta^\prime)\le\frac\beta2\|\theta-\theta^\prime\|^2$ for all $\theta,\theta^\prime\in\mathbb R^d$.
    \item[4.] $(\nabla f(\theta)-\nabla f(\theta^\prime))^\top(\theta-\theta^\prime)\ge\frac1\beta\|\nabla f(\theta)-\nabla f(\theta^\prime)\|^2$ for all $\theta,\theta^\prime\in\mathbb R^d$.
\end{itemize}
\end{lemma}
\begin{proof}
$~$\\
{\bf 1$\Rightarrow$2.} Let $g(\theta):=\frac{\beta}2\|\theta\|^2-f(\theta)$. Then for any $\theta\ne\theta^\prime$,
\begin{align*}
    (\nabla g(\theta)-\nabla g(\theta^\prime))^\top(\theta-\theta^\prime)&=\big(\beta(\theta-\theta^\prime)-(\nabla f(\theta)-\nabla f(\theta^\prime))\big)^\top(\theta-\theta^\prime)\\
    &=\beta\|\theta-\theta^\prime\|^2-(\nabla f(\theta)-\nabla f(\theta^\prime))^\top(\theta-\theta^\prime)\\
    &\ge \beta\|\theta-\theta^\prime\|^2-\|\nabla f(\theta)-\nabla f(\theta^\prime)\|\cdot\|\theta-\theta^\prime\|\\
    &\ge0.
\end{align*}
In addition, by the mean value theorem, there exists $s\in(0,1)$ such that for $\theta_s=s\theta+(1-s)\theta^\prime$,
$$g(\theta^\prime)-g(\theta)=\nabla g(\theta_s)^\top(\theta^\prime-\theta).$$
Using this, we can derive the following inequality
\begin{align*}
    0&\le(\nabla g(\theta_s)-\nabla g(\theta^\prime))^\top(\theta_s-\theta^\prime)\\
    &=s(\nabla g(\theta_s)-\nabla g(\theta^\prime))^\top(\theta-\theta^\prime)\\
    &=s(g(\theta)-g(\theta^\prime)-\nabla g(\theta^\prime)^\top(\theta-\theta^\prime))
    ,
\end{align*}
which implies that $g$ is convex.

\noindent{\bf 2$\Leftrightarrow$3.} The following relation shows the equivalence of the second and the third statements:
\begin{align*}
    &g(\theta)\ge g(\theta^\prime)+\nabla g(\theta^\prime)^\top(\theta-\theta^\prime)~\Leftrightarrow~ f(\theta)-f(\theta^\prime)-\nabla f(\theta^\prime)^\top(\theta-\theta^\prime)\le\frac\beta2\|\theta-\theta^\prime\|^2.
\end{align*}

\noindent{\bf 2,3$\Rightarrow$4.} 
We first introduce the following claim.
\begin{claim}\label{claim:smoothequiv}
Suppose that $f:\mathbb R^d\rightarrow\mathbb R$ is convex and differentiable and have a global minimum $\theta^*$. Then, $\frac1{2\beta}\|\nabla f(\theta)\|^2\le f(\theta)-f(\theta^*)$.
\end{claim}
\begin{proof}
The statement of Claim \ref{claim:smoothequiv} is a consequence of the following relation
$$f(\theta^*)=\inf_{\theta^\prime} f(\theta^\prime)\le\inf_{\theta^\prime}f(\theta)+\nabla f(\theta)^\top(\theta^\prime-\theta)+\frac{\beta}2\|\theta^\prime-\theta\|^2=f(\theta)-\frac1{2\beta}\|\nabla f(\theta)\|^2.$$
\end{proof}
Let $f_\theta(\phi):=f(\phi)-\nabla f(\theta)^\top\phi$. Since $\frac\beta2\|\phi\|^2-f_\theta(\phi)$ is convex and $\phi=\theta$ minimizes $f_\theta$, from Claim \ref{claim:smoothequiv}, we have
\begin{align*}
    &f(\theta^\prime)-f(\theta)-\nabla f(\theta)^\top(\theta^\prime-\theta)=f_{\theta}(\theta^\prime)-f_\theta(\theta)\ge\frac1{2\beta}\|\nabla f(\theta)-\nabla f(\theta^\prime)\|^2,\\
    &f(\theta)-f(\theta^\prime)-\nabla f(\theta^\prime)^\top(\theta-\theta^\prime)=f_{\theta^\prime}(\theta)-f_{\theta^\prime}(\theta^\prime)\ge\frac1{2\beta}\|\nabla f(\theta)-\nabla f(\theta^\prime)\|^2.
\end{align*}
Adding two above inequalities derive the fourth statement.

\noindent{\bf 4$\Rightarrow$1.}
The following inequality is sufficient for deriving the first statement
\begin{align*}
    &\frac1\beta\|\nabla f(\theta)-\nabla f(\theta^\prime)\|^2\le(\nabla f(\theta)-\nabla f(\theta^\prime))^\top(\theta-\theta^\prime)\le\|\nabla f(\theta)-\nabla f(\theta^\prime)\|\cdot\|\theta-\theta^\prime\|\\
    \Rightarrow&\|\nabla f(\theta)-\nabla f(\theta^\prime)\|\le\beta\|\theta-\theta^\prime\|.
\end{align*}
\end{proof}

\subsection{Contractivity of projected SGD for strongly convex and smooth objectives}\label{sec:contractivity}

\begin{lemma}\label{lem:strconvsmooth}
Let $f:\mathbb R^d\rightarrow\mathbb R$ be $\alpha$-strongly convex and $\beta$-smooth. Then for any $\eta\in(0,2/\beta)$ and for any convex $\Theta\subset\mathbb R^d$, $\theta\mapsto\Pi_\Theta(\theta-\eta\nabla f(\theta))$ is $\sqrt{1-2\alpha\eta+\alpha\beta\eta^2}$-contractive.
\end{lemma}
\begin{proof}
Let $g(\theta):=\theta-\eta\nabla f(\theta)$. Then for any $\theta,\theta^\prime\in\mathbb R^d$,
\begin{align*}
    &\|g(\theta)-g(\theta^\prime)\|^2=\|\theta-\theta^\prime\|^2-2\eta(\nabla f(\theta)-\nabla f(\theta^\prime))^\top(\theta-\theta^\prime)+\eta^2\|\nabla f(\theta)-\nabla f(\theta^\prime)\|^2\\
    &\le\|\theta-\theta^\prime\|^2-2\eta\left(\left(1-\frac{\beta\eta}2\right)\cdot\alpha\|\theta-\theta^\prime\|^2+\frac{\beta\eta}2\cdot\frac1\beta\|\nabla f(\theta)-\nabla f(\theta^\prime)\|^2\right)+\eta^2\|\nabla f(\theta)-\nabla f(\theta^\prime)\|^2\\
    &=(1-2\alpha\eta+\alpha\beta\eta^2)\|\theta-\theta^\prime\|^2.
\end{align*}
Here, the inequality is from the $\alpha$-strong convexity and Lemma~\ref{lem:smoothequiv}, i.e. 
\begin{align*}
    (\nabla f(\theta)-\nabla f(\theta^\prime))^\top(\theta-\theta^\prime)&\ge\alpha\|\theta-\theta^\prime\|^2\\
    (\nabla f(\theta)-\nabla f(\theta^\prime))^\top(\theta-\theta^\prime)&\ge\frac1\beta\|\nabla f(\theta)-\nabla f(\theta^\prime)\|^2.
\end{align*}
Using Lemma~\ref{lem:projection} completes the proof of Lemma~\ref{lem:strconvsmooth}.
\end{proof}

\clearpage

\section{Proofs of results in Section \ref{sec:mainresults}}
\subsection{Deriving generalization bounds using localized covers}
In Lemma~\ref{lem:cover}, we show that the union of all trajectories generated by $t\ge T=T_\varepsilon$ SGD iterations, i.e. $\bigcup_{\theta^{(0)}\in\Theta}\Psi_{\ge T}(\theta^{(0)})$, can be $\varepsilon$-covered by a set of points  generated by exactly $T$ SGD iterations, i.e. $\Psi_T(0)$.
Here, each $g_{i_T}\circ\cdots\circ g_{i_{1}}(0)\in\Psi_T(0)$ 
can be viewed as a \emph{deterministic} algorithm mapping samples $z_1,\dots,z_n$ to a parameter in $\Theta$, where each $g_{i_T}\circ\cdots\circ g_{i_{1}}(0)$ only depends on at most $T$ samples $z_{i_1},\dots,z_{i_T}$.
Under this observation, in this section, we generalize our localized cover so that its elements are general deterministic algorithms depending on a small number of samples, i.e. not restricted to possible instances of SGD.
Then, we derive generalization bounds using our (generalized) localized cover.
To this end, we define algorithms depending on samples.

\begin{mydefinition}\label{def:funcsamples}
$\phi:\mathcal Z^n\rightarrow\Theta$ is a ``(deterministic) algorithm depending on at most $T$ samples'' if there exists $\mathcal I\in\{\mathcal S\subset[n]:|\mathcal S|=T\}$ 
satisfying the following: for any $z_1,\dots,z_n,z_1^\prime,\dots,z_n^\prime\in\mathcal Z$ such that $z_i=z_i^\prime$ for all $i\in\mathcal I$, $\phi(z_1,\dots,z_n)=\phi(z_1^\prime,\dots,z_n^\prime)$.
Here, we refer to $\mathcal I$ as the ``set of sample indices determining $\phi$.''
The collection of all algorithms depending on at most $T$ samples is denoted by $\mathcal A_T$. 
\end{mydefinition}

In Definition~\ref{def:funcsamples}, we define $\phi$ to be an algorithm depending on at most $T$ samples if the value of $\phi$ can be fully determined by a subset of the samples $\{z_i\}_{i\in\mathcal I}$ for some $\mathcal I\in\{\mathcal S\subset[n]:|\mathcal S|=T\}$. We note that $\Psi_T(\theta^{(0)})\subset\mathcal A_T$ for any $T\ge0$ and $\theta^{(0)}\in\Theta$.

Now, we introduce the following theorem for deriving generalization bounds using a localized cover where each element in the cover is an algorithm depending on at most $T$ samples.
\begin{theorem}\label{thm:covering}
Suppose that Assumptions \ref{asm:lipschitz} \& \ref{asm:boundedval} hold and $\Theta\subset\mathbb R^d$.
Let $\varepsilon>0$, $\Phi_n\subseteq\mathcal A_n$, and $\Phi_{T,\varepsilon}\subseteq\mathcal A_T$
such that 
\begin{align}
    \psi(x_{1:n})\in\bigcup_{\phi\in\Phi_{T,\varepsilon}}\mathcal B_\varepsilon^d\big(\phi(x_{1:n})\big)\quad\text{for all $x_{1:n}\in\mathcal Z^n$ and $\psi\in\Phi_{n}$}.\label{eq:coveringasm}
\end{align}
Let $\mu$ be a distribution over $\mathcal Z$ and $z_{1:n}=(z_1,\dots,z_n)$
is such that $z_i$'s are i.i.d.\ samples from $\mu$. 
Then, with probability at least $1-\delta$ over the sampling distribution of $z_{1:n}$, for any $\psi\in\Psi$,
\begin{align*}
    \left|\hat F\big(\psi(z_{1:n})\big)-F\big(\psi(z_{1:n})\big)\right|\le\frac{BT}n+B\sqrt{\frac{\log(2|\Phi_{T,\varepsilon}|/\delta)}{2n}}+2L\varepsilon.
\end{align*}
\end{theorem}
Theorem~\ref{thm:covering} implies that if (i) $\Phi_n$ can be covered by a set $\Phi_{T,\varepsilon}$ of algorithms depending on at most $T$ samples and (ii) $B$, $L$, $T$, and $|\Phi_{T,\varepsilon}|$ are independent of $d$, then a dimension-independent  generalization bound can be derived. We note that the assumption~\eqref{eq:coveringasm} is a generalization of the observation in Lemma~\ref{lem:cover} since $\Psi_{\ge T}\subset\mathcal A_n$ and $\Psi_{T}\subset\mathcal A_T$.
\begin{proof}[Proof of Theorem~\ref{thm:covering}]
From Assumption~\ref{asm:lipschitz} and the assumption \eqref{eq:coveringasm}, for any $\psi\in\Phi_n$, we have
\begin{align}
    &\left|\hat F\big(\psi(z_{1:n})\big)-F\big(\psi(z_{1:n})\big)\right|\le\sup_{\phi\in\Phi_{T,\varepsilon}}
    \left|\hat F\big(\phi(z_{1:n})\big)-F\big(\phi(z_{1:n})\big)\right|+2L\varepsilon.\label{eq:pfthm:covering-1}
\end{align}
Namely, if we bound $|\hat F(\phi(z_{1:n}))-F(\phi(z_{1:n}))|$ for all $\phi\in\Phi_{T,\varepsilon}$, the bound in Theorem~\ref{thm:covering} follows. Since the statement of Theorem \ref{thm:covering} is trivial if $T=n$ or $|\Phi_{T,\varepsilon}|=\infty$, we assume $T<n$ and $|\Phi_{T,\varepsilon}|<\infty$. Now, we derive the target bound: for $\phi\in\Phi_{T,\varepsilon}$, $\theta:=\phi(z_{1:n})\in \Theta$, and the set of indices $\mathcal I_\phi$ determining $\phi$ with $|\mathcal I_\phi|\le T$, 
\begin{align}
    \left|\hat F(\theta)-F(\theta)\right|
    &\le\Bigg|\frac1n\sum_{i\in\mathcal I_\phi}f(\theta;z_i)-F(\theta)\Bigg|+\Bigg|\frac1n\sum_{i\in[n]\setminus\mathcal I_\phi}f(\theta;z_i)-F(\theta)\Bigg|\notag\\
    &\le\frac{B|\mathcal I_\phi|}{n}+\Bigg|\frac1n\sum_{i\in[n]\setminus\mathcal I_\phi}f(\theta;z_i)-F(\theta)\Bigg|\notag\\
    &\le\frac{B|\mathcal I_\phi|}{n}+\frac{n-|\mathcal I_\phi|}{n}\sqrt{\frac{B^2\log(2|\Phi_{T,\varepsilon}|/\delta)}{2(n-|\mathcal I_\phi|)}}\quad\text{w.p.}\quad1-\delta/|\Phi_{T,\varepsilon}|\notag\\
    &\le\frac{BT}{n}+B\sqrt{\frac{\log(2|\Phi_{T,\varepsilon}|/\delta)}{2n}}\quad\text{w.p.}\quad1-\delta/|\Phi_{T,\varepsilon}|.\label{eq:pfthm:covering-2}
\end{align}
For the first inequality in \eqref{eq:pfthm:covering-2}, we use the triangle inequality to upper bound $|\hat F(\theta)-F(\theta)|$ using two terms: we only utilize the samples determining $\theta$ in the first term while remaining samples independent of $\theta$ are considered in the second term. 
{The second inequality directly follows from Assumption~\ref{asm:boundedval}.
To bound the second term in RHS of the second inequality, one can apply the Hoeffding's inequality (see Lemma~\ref{lem:boundedtail}), which leads us to the third inequality in \eqref{eq:pfthm:covering-2}. Here, the last inequality naturally follows.}
By using \eqref{eq:pfthm:covering-1}, \eqref{eq:pfthm:covering-2}, and by applying the union bound for all $\phi\in\Phi_{T,\varepsilon}$, we obtain the bound in Theorem~\ref{thm:covering}.
\end{proof}

\subsection{Proof of Theorem~\ref{thm:strconvsmooth}}\label{sec:pfthm:strconvsmooth}
First, observe that $g_1,\dots,g_n$ are $\gamma$-contractive by Lemma~\ref{lem:strconvsmooth}.
Let $\varepsilon=1/(2Ln)$; then by Lemma~\ref{lem:cover}, we have
\begin{align}
    \bigcup_{\theta^{(0)}\in\Theta}\Psi_{\geq T}(\theta^{(0)}) \subseteq \bigcup_{\phi \in \Psi_T(0)}\mathcal B_\varepsilon^d(\phi).\label{eq:pfthm:strconvsmooth}
\end{align}
Now, we apply Theorem~\ref{thm:covering} with
\begin{align*}
&\Phi_n\leftarrow\bigcup_{\theta^{(0)}\in\Theta}\Psi_{\ge T}(\theta^{(0)}),~\Phi_{T,\varepsilon}\leftarrow\Psi_T(0),~\varepsilon\leftarrow\varepsilon,~T\leftarrow T,~\delta\leftarrow\delta.
\end{align*}
where the assumption \eqref{eq:coveringasm} in Theorem~\ref{thm:covering} is satisfied by \eqref{eq:pfthm:strconvsmooth}. This completes the proof of  Theorem~\ref{thm:strconvsmooth}.

\subsection{Proof of Corollary \ref{cor:strconvsmooth}}\label{sec:pfcor:strconvsmooth}
Let $\theta^{(t)}:=g_{i_t}\circ\cdots\circ g_{i_1}(\theta^{(0)})$ and
$\phi:=g_{i_t}\circ\cdots\circ g_{t-T+1}(0)$, i.e. $\phi$ is an algorithm depending on at most $T$ samples. Since each $g_{i}$ is $\gamma$-contractive by Lemma~\ref{lem:strconvsmooth}, one can observe that 
\begin{align}
\|\theta^{(t)}-\phi\|\le\gamma^TR\le\frac1{2Ln}=:\varepsilon. \label{eq:pfcor:strconvsmooth}
\end{align}
Now, we apply Theorem~\ref{thm:covering} with
\begin{align*}
&\Phi_n\leftarrow\{\theta^{(t)}\},~\Phi_{T,\varepsilon}\leftarrow\{\phi\},~\varepsilon\leftarrow\varepsilon,~T\leftarrow T,~\delta\leftarrow\delta.
\end{align*}
where the assumption \eqref{eq:coveringasm} is satisfied by \eqref{eq:pfcor:strconvsmooth} and $|\Phi_\varepsilon|=1$. This provides the bound in Corollary~\ref{cor:strconvsmooth}.
\subsection{Proof of Corollary~\ref{thm:early}}\label{sec:pfthm:early}
The proof of Corollary~\ref{thm:early} is simple. 
Since $\theta^{(0)}$, $t$, and $i_1,\dots,i_t$ are fixed, $\theta^{(t)}=g_{i_t}\circ\cdots\circ g_{i_1}(\theta^{(0)})$ is an algorithm depending on at most $t$ samples. Then by using Theorem~\ref{thm:covering} with $$\Phi_n,\Phi_{T,\varepsilon}\leftarrow\{\theta^{(t)}\},~\varepsilon\leftarrow0,~T\leftarrow t,~\delta\leftarrow\delta,$$
we obtain the bound in Corollary~\ref{thm:early}.

\subsection{Derivation of \eqref{eq:fractalbound}}\label{sec:pfeq:fractalbound}
We first formally define the Hausdorff dimension.
\begin{mydefinition}
Given $\mu\subset\mathbb R^d$, $d_H$ defined below is the ``Hausdorff dimension'' of $\mu$:
\begin{align*}
    d_H:=&\inf\{s\ge0:h^s(\mu)=0\},\\
    h^s(\mu):=&\lim_{r\rightarrow0}\inf\Big\{\sum_{i=1}^kr_i^s:k\in\mathbb N\cup\{\infty\},(r_i)_{i=1}^k\in(0,r)^k\\
    &\qquad\qquad\qquad\qquad\text{~such that there exists $(\theta_i)_{i=1}^k$ satisfying $\mu\subset\bigcup_i\mathcal B_{r_i}(\theta_i)$}\Big\}.
\end{align*}
\end{mydefinition}

From the assumption on $f(\theta;z)$, one can observe that for all $i\in[n]$ and $\theta,\theta^\prime\in\mathcal B_R^d(0)$, 
\begin{align*}
    \|g_i(\theta)-g_i(\theta^\prime)\|\le\gamma\|\theta-\theta^\prime\|
\end{align*}
where $\gamma=1-\eta$ since $\alpha=\beta=1$.
Then, the assumption $\|\theta^*_{z_i}-\theta^*_{z_j}\|\ge2\gamma R$ guarantees that $g_i(\text{\rm int}(\mathcal B_R))\cap g_j(\text{\rm int}(\mathcal B_R))=\emptyset$ where $\text{\rm int}(\mathcal S)$ denotes the interior of a set $\mathcal S$.

Finally, the following theorem shows that $d_H=\frac{\log n}{\log(1/\gamma)}$. Substituting $d_H$ to the bound in Theorem \ref{thm:strconvsmooth} results in
\eqref{eq:fractalbound}.
\begin{theorem}[Theorem 9.3 in \citep{falconer14}]
Suppose that $g_i(\text{\rm int}(\mathcal B_R))\cap g_j(\text{\rm int}(\mathcal B_R))=\emptyset$ for all $i\ne j$ and $\|g_i(\theta)-g_i(\theta^\prime)\|=\gamma\|\theta-\theta^\prime\|$ for all $i$.
Then $d_H=\frac{\log n}{\log(1/\gamma)}$.
\end{theorem}

\subsection{Proof of Proposition~\ref{prop:smoothapprox}}\label{sec:pfprop:smoothapprox}
In this proof, we explicitly construct a piecewise $\beta$-strongly convex and $\beta$-smooth function, i.e. quadratic, which is always piecewise $\alpha$-strongly convex and $\beta$-smooth for any $\alpha\in[0,\beta]$.
Let $\{\mathcal Q_1,\dots,\mathcal Q_Q\}$ be the partition of $\mathcal B_R^d(0)$ such that $f=\ell_q$ on $\mathcal Q_q$ for some smooth $\ell_q$ on $\mathcal B_R^d(0)$.
Given $q\in[Q]$ and $\varepsilon>0$, let $\mathcal C_{q,\varepsilon}=\{\phi_{q,1},\dots,\phi_{q,P_q}\}\subset\Theta$ be an $\varepsilon$-cover of $\Theta$ with the minimum cardinality, i.e. $P_q\le (3R/\varepsilon)^d$ for $\varepsilon\le R$. Let $\{\mathcal P_{q,1},\dots,\mathcal P_{q,P_q}\}$ be a partition of $\mathcal Q_q$ where each $\mathcal P_{q,p}$ is defined as follows:
\begin{align*}
\mathcal P_{q,p}:=\left\{\theta\in\mathcal Q_q\setminus\bigcup_{r<p}\mathcal P_{q,r}:\|\phi_{q,p}-\theta\|\le\min_{r>p}\|\phi_{q,r}-\theta\|\right\}.
\end{align*}
For each $q\in[Q]$ and $p\in[P_q]$, we also define $h(\theta):=h_{q,p}(\theta)$ on $\mathcal P_{q,P_q}$ where
\begin{align*}
    h_{q,p}(\theta):=&\phi_{q,p}+\nabla \ell_q(\phi_{q,p})^\top(\theta-\phi_{q,p})+\frac\beta2\|\theta-\phi_{q,p}\|^2
\end{align*}
and $\nabla_j$ denotes the partial derivative with respect to the $j$-th entry. Then, one can observe that $h$  is piecewise $\alpha$-strongly convex and $\beta$-smooth with $Q(3R/\varepsilon)^d$ pieces for any $\alpha\le\beta$. Furthermore, for any $\theta\in\mathcal P_{q,p}$, we have
\begin{align*}
    \|\nabla f(\theta)-\nabla h(\theta)\|=\|\nabla \ell_q(\theta)-\nabla \ell_q(\phi_{q,p})-\beta(\theta-\phi_{q,p})\|\le(\beta+\beta^\prime)\varepsilon.
\end{align*}
Choosing $\varepsilon:=\xi/(\beta+\beta^\prime)$ completes the proof of Proposition~\ref{prop:smoothapprox}.

\subsection{Proof of Theorem~\ref{thm:approxpiecestrconvsmooth}}\label{sec:pfthm:approxpiecestrconvsmooth}
Given $z\in\mathcal Z$, let $\mathcal P_{z,1},\dots,\mathcal P_{z,P}$ be a partition of $\Theta$ and $h_1(\,\cdot\,;z),\dots,h_P(\,\cdot\,;z)$ be $\alpha$-strongly convex and $\beta$-smooth functions such that $h(\,\cdot\,;z)=h_p(\,\cdot\,;z)$ on $\mathcal P_{z,p}$ for all $p\in[P]$.
Let $g_{i,p}(\theta):=\theta-\eta\nabla h_{p}(\theta;z_i)$, i.e. each $g_{i,p}$ is $\gamma$-contractive by Lemma~\ref{lem:strconvsmooth}.
Given $\psi^{(0)}\in\Theta$, $t\ge T$, and $i_1,\dots,i_t\in[n]$, let $\psi^{(s)}:=g_{i_s}\circ\cdots\circ g_{i_1}(\psi^{(0)})$ for all $s\in[t]$ and let $p_{s}\in[P]$ be an index satisfying $\psi^{(s-1)}\in\mathcal P_{z_{i_s},p_{s}}$.
Let $\phi:=g_{i_t,p_{t}}\circ\cdots\circ g_{i_{t-T+1},p_{t-T+1}}(0)$, i.e. $\phi$ is an algorithm depending on at most $T$ samples. Then for any $\theta\in\Theta$, we have
\begin{align*}
    \|\psi^{(s)}-&g_{i_s,p_{s}}(\theta)\|=\|\psi^{(s-1)}-\eta\nabla f(\psi^{(s-1)};z_{i_s})-g_{i_s,p_{s}}(\theta)\|\\
    &=\|\psi^{(s-1)}-\eta\nabla h(\psi^{(s-1)},z_{i_s})+\eta\nabla h(\psi^{(s-1)},z_{i_s})-\eta\nabla f(\psi^{(s-1)};z_{i_s})-g_{i_s,p_{s}}(\theta)\|\\
    &\le\|g_{i_s,p_{s}}(\psi^{(s-1)})-g_{i_s,p_{s}}(\theta)\|+\eta\|\nabla f(\psi^{(s-1)};z_{i_s})-\nabla h(\psi^{(s-1)};z_{i_s})\|\\
    &\le\gamma\|\psi^{(s-1)}-\theta\|+\eta\xi
\end{align*}
This implies that
\begin{align}
    \|\psi^{(t)}-\phi\|\le\gamma^T\|\psi^{(0)}\|+\eta\xi\sum_{t=0}^{T-1}\gamma^t\le\gamma^{T}R+\frac{1-\gamma^T}{1-\gamma}\eta\xi=:\varepsilon.\label{eq:pfthm:approxpiecestrconvsmooth}
\end{align}
Now, we apply Theorem~\ref{thm:covering} with
\begin{align*}
&\Phi_n\leftarrow\bigcup_{\theta^{(0)}\in\Theta}\Psi_{\ge T}(\theta^{(0)}),~\Phi_{T,\varepsilon}\leftarrow\{g_{i_T,p_{T}}\circ\cdots\circ g_{i_1,p_1}(0):i_1,\dots,i_T\in[n],p_1,\dots,p_T\in[P] \},\\
&\varepsilon\leftarrow\varepsilon,~T\leftarrow T,~\delta\leftarrow\delta
\end{align*}
where the assumption \eqref{eq:coveringasm} in Theorem~\ref{thm:covering} is satisfied by \eqref{eq:pfthm:approxpiecestrconvsmooth} and $|\Phi_{T,\varepsilon}|\le(nP)^T$. This leads us to the bound in Theorem~\ref{thm:approxpiecestrconvsmooth}.

\clearpage

\section{Proofs of results in Section \ref{sec:mainapplications}}\label{sec:pfthms:mainapplications}
\subsection{Proof of Theorem~\ref{thm:glm}}\label{sec:pfthm:glm}
Given $z=(y,x)\in\mathcal Y\times\mathcal B_{R_x}^d(0)=\mathcal Z$, let $\{\mathcal Q_{z,1},\dots,\mathcal Q_{z,Q}\}$ be a partition of $\Theta$ such that $\ell(\,\cdot\,,\dots,\,\cdot\,;y)$ is $\beta$-smooth on $\mathcal Q_{z,q}$ for all $q\in[Q]$.
For proving Theorem~\ref{thm:glm}, we utilize Theorem~\ref{thm:approxpiecestrconvsmooth} by approximating each $f(\,\cdot\,;z)$ using a piecewise strongly convex and smooth function $h(\,\cdot\,;z)$ with $P^KQ$ pieces. 
To this end, we first define $P$ by
\begin{align*}
    P:=&\left\lceil\frac{2R R_x}{\kappa}\right\rceil\quad\text{where}\quad\kappa:=\frac1{12\beta \eta KLR_xTn}.
\end{align*}
Using $P$ defined as above, for each $z=(y,x)\in\mathcal Y\times\mathcal B_{R_x}^d(0)$,
we construct a partition $\{\mathcal P_{z,q,p_1,\dots,p_K}:p_1,\dots,p_K\in[P],q\in[Q]\}$ of $\Theta$ as follows
\begin{align*}
    &\mathcal P_{z,q,p_1,\dots,p_K}:=    \{(\theta_j)_{j=1}^K:(\theta_1^\top x,\dots,\theta_K^\top x)\in\mathcal Q_{z,q},\theta_j^\top x\in\mathcal T_{p_j}~\forall j\in[K]\}
\end{align*}
where
\begin{align*}
    \mathcal T_{p}:=&\begin{cases}
    [\mu_p,\mu_{p+1})\quad&\text{if}~p\in[P-1]\\
    [\mu_P,\mu_{P+1}]\quad&\text{if}~p=P
    \end{cases},\\
    \mu_{p}:=&-R R_x+(p-1)\kappa\quad\forall p\in[P+1].
\end{align*}
For the notational simplicity, let $u:=(z,q,p_1,\dots,p_K)\in\mathcal Z\times[Q]\times[P]^K$. Now, we define our approximation as $h(\,\cdot\,;z)=h_{u}(\,\cdot\,)$ on $\mathcal P_{u}$ if $\mathcal P_{u}\ne\emptyset$. Here, $h_{u}(\,\cdot\,)$ is defined as follows: for some fixed $\nu_u:=(\nu_{u,1},\dots,\nu_{u,K})\in\mathcal P_{u}$,
\begin{align*}
h_{u}(\theta):=\ell(\nu_{u,1}^\top x,\dots,\nu_{u,K}^\top x;y)+\sum_{j=1}^K\nabla_j\ell(\nu_{u,1}^\top x,\dots,\nu_{u,K}^\top x;y)(\theta_j^\top x-\nu_{u,j}^\top x)+\sum_{j=1}^K\frac\lambda2\|\theta_j\|^2.
\end{align*}
where $\nabla_j$ denotes the partial derivative with respect to the $j$-th argument. 
Namely, $h_{u}$ is a first order approximation of $\ell$ at $\theta$ with $\ell_2$-regularization; hence, $h_{u}$ is $\lambda$-strongly convex and $\lambda$-smooth. 
Then given $z\in\mathcal Z$ and $\theta\in\Theta$, for $q$ and $p_1,\dots,p_K$ satisfying $\theta\in\mathcal P_{u}$ for $u=(z,q,p_1,\dots,p_K)$, one can observe that
\begin{align*}
    \|\nabla f(\theta;z)-\nabla h_{u}(\theta)\|&\le K^{1/2}\max_{j\in[K]}\|\nabla_j\ell(\theta_1^\top x,\dots,\theta_K^\top x;y)x-\nabla_j\ell(\nu_{u,1}^\top x,\dots,\nu_{u,K}^\top x;y)x\|\\
    &\le\beta\kappa KR_x=\frac{1}{12\eta LTn}.
    %\beta\kappa K^{3/2}R_x.
\end{align*}
Now, we apply Theorem~\ref{thm:approxpiecestrconvsmooth} with 
\begin{align*}
h\leftarrow h,~P\leftarrow P^KQ,~\xi\leftarrow\frac{1}{12\eta LTn},
~T\leftarrow T,~\alpha\leftarrow\lambda,~\beta\leftarrow\lambda,~L\leftarrow L,~B\leftarrow B.
\end{align*}
Then in the bound of Theorem~\ref{thm:approxpiecestrconvsmooth}, we have
\begin{align*}
    2L\bigg(\gamma^{T}R+{\color{black}\frac{1-\gamma^T}{1-\gamma}}\eta\xi\bigg)\le\frac1n
\end{align*}
since $\gamma^TR\le1/(3Ln)$ and $\frac{1-\gamma^T}{1-\gamma}=\sum_{j=0}^{T-1}\gamma^{j}\le T$.
This completes the proof of Theorem~\ref{thm:glm}.

\subsection{Equivalence between soft $K$-means algorithm and expectation-maximization for \eqref{eq:softkmeansem}}\label{sec:softkmeansem}
We first observe that applying the affine transformation $x\mapsto -nKx+\log(\frac{(\xi/\pi)^{d/2}}{K})$ to \eqref{eq:softkmeansem} results in the following objective:
\begin{align}
    \sum_{i=1}^n\log\left(\frac1K(\zeta/\pi)^{d/2}\sum_{j=1}^K\exp\left(-\zeta\|\theta_j-z_i\|^2\right)\right)\label{eq:gmmll}
\end{align}
i.e. the expectation-maximization algorithm for \eqref{eq:gmmll} aims to find a local minimum of \eqref{eq:softkmeansem}.
Since \eqref{eq:gmmll} is the log-likelihood for the mixture of Gaussians $N(\theta_1,\frac1{2\zeta}I),\dots,N(\theta_K,\frac1{2\zeta}I)$ under the same cluster density $1/K$ with observations $z_1,\dots,z_n$, the expectation-maximization algorithm for \eqref{eq:gmmll} is identical to the alternative procedure~\eqref{eq:softkmeansalg} \citep{prince12}. 

\subsection{Proof of Theorem~\ref{thm:softkmeans}}\label{sec:pfthm:softkmeans}
First, observe that for $\theta\in(\mathcal B_R^d(0))^K,z\in\mathcal B_R^d(0)$, we have
\begin{align*}
\frac{\partial}{\partial\theta_k}f(\theta;z)&=\frac{2(\theta_k-z_i)\exp(-\zeta\|\theta_k-z_i\|^2)}{\sum_{j=1}^K\exp(-\zeta\|\theta_j-z_i\|^2)},\\
\frac{\partial^2}{\partial\theta_k\partial \theta_{k^\prime}}f(\theta;z)&=\frac{4\zeta(\theta_k-z_i)(\theta_{k^\prime}-z_i)^\top\exp(-\zeta\|\theta_k-z_i\|^2-\zeta\|\theta_{k^\prime}-z_i\|^2)}{\big(\sum_{j=1}^K\exp(-\zeta\|\theta_j-z_i\|^2)\big)^2}\\
&\qquad+\mathbf1_{k=k^\prime}\times\frac{(2I-4\zeta(\theta_k-z_i)(\theta_k-z_i)^\top)\exp(-\zeta\|\theta_k-z_i\|^2)}{\sum_{j=1}^K\exp(-\zeta\|\theta_j-z_i\|^2)}
\end{align*}
where $\mathbf 1_{k=k^\prime}$ is one if $k=k^\prime$ and zero otherwise.
Using this, we bound $f$ and its first and second derivatives as follows:
\begin{align*}
f(\theta;\,&z)\in[-B+\log K,\log K],\\
L:=&\frac{4R}{\sqrt{K}}\exp(\zeta B)=4R\sqrt{\frac{K}{K^2\exp(-\zeta B)^2}}\ge4R\sqrt{\frac{\sum_{j=1}^K\exp(-\zeta\|\theta_j-z_i\|^2)^2}{\big(\sum_{j=1}^K\exp(-\zeta\|\theta_j-z_i\|^2)\big)^2}}\\
\ge&\sqrt{\sum_{k=1}^K\left\|\frac{\partial}{\partial\theta_k}f(\theta;z)\right\|^2}=\left\|\frac{\partial}{\partial\theta}f(\theta;z)\right\|,\\
\alpha:=&\frac{2}K\exp\left(-\zeta B\right)\le\frac1{\|\theta_j-z\|}\left\|\frac{\partial}{\partial\theta_j}f(\theta;z)\right\|\quad\forall \theta~~\text{s.t.}~~\theta_j\ne z,\\
\beta:=&\frac{2}K\exp\left(\zeta B\right)\ge\frac1{\|\theta_j-z\|}\left\|\frac{\partial}{\partial\theta_j}f(\theta;z)\right\|\quad\forall \theta~~\text{s.t.}~~\theta_j\ne z,\\
\beta^\prime:=&4\zeta B\exp({\zeta B})+4\zeta B+2\ge\left\|\nabla^2 f(\theta;z)\right\|_2.
\end{align*}
To utilize Theorem~\ref{thm:approxpiecestrconvsmooth}, we approximate each $f(\,\cdot\,;z)$ using a piecewise strongly convex and smooth function $h(\,\cdot\,;z)$ with $P^K$ pieces. 
To this end, we first construct a partition $\{\mathcal P_{z,p_1,\dots,p_K}:p_1,\dots,p_K\in[P]\}$ of $\Theta$ for each $z\in\mathcal Z$ as follows:
\begin{align*}
    \mathcal P_{z,p_1,\dots,p_K}:=&%&\begin{cases}
    \{(\theta_j)_{j=1}^K\in\Theta:\|\theta_j-z\|\in\mathcal T_{p_j}~\forall j\in[K]\}
\end{align*}
where
\begin{align*}
    \kappa:=&\frac1{12(\beta+\beta^\prime)\eta \sqrt{K}LTn},\\
    P:=&\left\lceil\frac{2R}\kappa\right\rceil,\\
    \mu_p:=&(p-1)\kappa\quad\forall p\in[P+1],\\
    \mathcal T_{p}:=&\begin{cases}
    [\mu_p,\mu_{p+1})\quad&\text{if}~p\in[P-1]\\
    [\mu_P,\mu_{P+1}]\quad&\text{if}~p=P
    \end{cases}.
\end{align*}
We define $h(\,\cdot\,;z):=h_{z,p_1,\dots,p_K}(\,\cdot\,)$ on $\mathcal P_{z,p_1,\dots,p_K}$ if $\mathcal P_{z,p_1,\dots,p_K}\ne\emptyset$.
Here, note that $\mathcal P_{z,p_1,\dots,p_K}\setminus\{z\}\ne\emptyset$ if $\mathcal P_{z,p_1,\dots,p_K}\ne\emptyset$ from the definition of $\mathcal P_{z,p_1,\dots,p_K}$.
Now, we define $h_{z,p_1,\dots,p_K}(\,\cdot\,)$ as follows: for some fixed $\nu_{z,p_1,\dots,p_K}=((\nu_{z,p_1,\dots,p_K})_j)_{j=1}^K\in\mathcal P_{z,p_1,\dots,p_K}\setminus\{z\}$,
\begin{align*}
    h_{z,p_1,\dots,p_K}(\theta):=\sum_{j=1}^K\frac12\left\|\frac1{\|(\nu_{z,p_1,\dots,p_K})_j-z\|}\frac{\partial}{\partial\theta_j^\prime}f(\theta^\prime;z)\Big|_{\theta^\prime=\nu_{z,p_1,\dots,p_K}}\right\|\cdot\|\theta_j-z\|^2.
\end{align*}
Our construction of $h_{z,p_1,\dots,p_K}$ has some nice properties. For example, for
\begin{align*}
    &a_{z,p_1,\dots,p_K}:=\min_j\left\|\frac1{\|(\nu_{z,p_1,\dots,p_K})_j-z\|}\frac{\partial}{\partial\theta_j^\prime}f(\theta^\prime;z)\Big|_{\theta^\prime=\nu_{z,p_1,\dots,p_K}}\right\|,\\
    &b_{z,p_1,\dots,p_K}:=\max_j\left\|\frac1{\|(\nu_{z,p_1,\dots,p_K})_j-z\|}\frac{\partial}{\partial\theta_j^\prime}f(\theta^\prime;z)\Big|_{\theta^\prime=\nu_{z,p_1,\dots,p_K}}\right\|,
\end{align*}
$h_{z,p_1,\dots,p_K}$ is $a_{z,p_1,\dots,p_K}$-strongly convex and $b_{z,p_1,\dots,p_K}$-smooth and $\alpha\le a_{z,p_1,\dots,p_K}\le b_{z,p_1,\dots,p_K}\le\beta$.
Furthermore, for any $\nu^\prime=(\nu^\prime_j)_{j=1}^K$ satisfying 
\begin{align}
\|\nu^\prime_j-z\|=\|(\nu_{z,p_1,\dots,p_K})_j-z\|,\label{eq:pfthm:softkmeans}
\end{align}
we have
\begin{align}
    \nabla f(\nu^\prime,z)=\nabla h_{z,p_1,\dots,p_K}(\nu^\prime)\label{eq:pfthm:softkmeans2}
\end{align}
from the symmetry of $f$ and $h_{z,p_1,\dots,p_K}$.

Now, we bound $\|\nabla f(\theta;z)-\nabla f(\theta;z)\|$ to utilize Theorem~\ref{thm:approxpiecestrconvsmooth}.
For $\theta=(\theta_j)_{j=1}^K\in(\mathcal B_R^d(0))^K$, let $p_1,\dots,p_K$ be indices in $[P]$ satisfying $\theta\in\mathcal P_{z,p_1,\dots,p_K}$. Let $\nu^\prime$ be a point on the line connecting $z$ and $\theta$ such that $\nu^\prime$ satisfies \eqref{eq:pfthm:softkmeans}, i.e. $\|\nu_j^\prime-\theta_j\|\le\kappa$ for all $j\in[K]$. Then, we have
\begin{align*}
    \|\nabla f(\theta;z)-\nabla h(\theta;z)\|&=\|\nabla f(\theta;z)-\nabla f(\nu^\prime;z)+\nabla f(\nu^\prime;z)-\nabla h(\theta;z)\|\\
    &\le\|\nabla f(\theta;z)-\nabla f(\nu^\prime;z)\|+\|\nabla h(\theta;z)-\nabla h_{z,p_1,\dots,p_K}(\nu^\prime)\|\\
    &\le \beta^\prime\kappa\sqrt{K}+\sqrt{K}\max_{j\in[K]}\left\{\left\|\frac{\partial}{\partial\theta_j}h_{z,p_1,\dots,p_K}(\theta)-\frac{\partial}{\partial\theta_j}h_{z,p_1,\dots,p_K}(\nu^\prime)\right\|\right\}\\
    &\le (\beta+\beta^\prime)\kappa\sqrt{K}\\
    &\le\frac1{12\eta LTn}
\end{align*}
The first inequality holds since $\nabla f(\nu^\prime;z)=\nabla h_{z,p_1,\dots,p_K}(\nu^\prime)$ by \eqref{eq:pfthm:softkmeans2} and other inequalities holds from the definitions of $\beta$ and $\beta^\prime$. 
Now, we apply Theorem~\ref{thm:approxpiecestrconvsmooth} with
\begin{align*}
h\leftarrow h,~P\leftarrow P^K,~\xi\leftarrow \frac1{12\eta LTn},~T\leftarrow T,~\alpha\leftarrow\alpha,~\beta\leftarrow\beta,~L\leftarrow L,~B\leftarrow B.
\end{align*}
Then in the bound of Theorem~\ref{thm:approxpiecestrconvsmooth}, we have
\begin{align*}
    2L\bigg(\gamma^{T}R+{\color{black}\frac{1-\gamma^T}{1-\gamma}}\eta\xi\bigg)\le\frac1n
\end{align*}
since $\gamma^TR\le1/(3Ln)$ and $\frac{1-\gamma^T}{1-\gamma}=\sum_{j=0}^{T-1}\gamma^{j}\le T$.
This completes the proof of Theorem~\ref{thm:softkmeans}.

\subsection{Proof of Theorem~\ref{thm:hardkmeans}}\label{sec:pfthm:hardkmeans}
For $i\in[n]$ and $j\in[K]$, we first define $g_{i,j}$ as
\begin{align*}
    g_{i,j}(\theta):=&(\theta^\prime_k)_{k=1}^K\quad\text{where}\quad\theta^\prime_k=\begin{cases}
    \theta_k-2\eta(\theta_k-z_i)~&\text{if}~k=j\\
    \theta_k~&\text{if}~k\ne j
    \end{cases}.
\end{align*}
For $\varepsilon:=\frac1{8Rn}$, We define
\begin{align*}
    &\Phi_\varepsilon:=\{(\theta^\prime_j)_{j=1}^K:\theta^\prime_j=\big(g_{i_{j,t_j},j}\circ\cdots\circ g_{i_{j,1},j}(\theta^{(0)})\big)_j,t_j\in[T]\cup\{0\},i_{j,1},\dots,i_{j,t_j}\in[n],~\forall j\in[K]\}.\notag
\end{align*}
Given $\theta^{(0)}$, since $g_{i,j}$ is $\gamma$-contractive by Lemma~\ref{lem:strconvsmooth}, we have for any $t\ge0$, $i_1,\dots,i_t\in[n]$, and $j\in[K]$,
\begin{align*}
    \big(g_{i_t}\circ\cdots\circ g_{i_1}(\theta^{(0)})\big)_j\in\bigcup_{\phi\in\Phi_\varepsilon}\mathcal B_{K^{-1/2}\varepsilon}^d(\phi_j),%\label{eq:pfthm:hardkmeans}
\end{align*}
i.e. 
\begin{align}
    g_{i_t}\circ\cdots\circ g_{i_1}(\theta^{(0)})\in\bigcup_{\phi\in\Phi_\varepsilon}\mathcal B_{\varepsilon}^d(\phi),\label{eq:pfthm:hardkmeans}
\end{align}
Namely, $\Phi_\varepsilon$ contains all possible SGD parameters starting from $\theta^{(0)}$ where each element in $\Phi_\varepsilon$ is an algorithm depending on at most $KT$ samples.
Since each $f(\,\cdot\,;z)$ is $(4R)$-Lipschitz on $\Theta$,
we apply Theorem~\ref{thm:covering} with
\begin{align*}
&\Phi_n\leftarrow\{g_{i_t}\circ\cdots\circ g_{i_1}(\theta^{(0)}):t\ge0,i_1,\dots,i_t\in[n]\}\\
&\Phi_{T,\varepsilon}\leftarrow\Phi_\varepsilon,~\varepsilon\leftarrow\varepsilon,~T\leftarrow KT,~\delta\leftarrow\delta,~L\leftarrow4R,~B\leftarrow B.
\end{align*}
where 
\begin{align*}
|\Phi_\varepsilon|&\le\sum_{t_1,\dots,t_K\in[T]\cup\{0\}}n^{\sum_{k=1}^K t_k}\le\sum_{t_1,\dots,t_K\in[T]\cup\{0\}}n^{KT}\le(T+1)^K\cdot n^{KT}\le(2n)^{KT}
\end{align*}
and the assumption \eqref{eq:coveringasm} in Theorem~\ref{thm:covering} is satisfied by \eqref{eq:pfthm:hardkmeans}. This leads us to the bound in Theorem~\ref{thm:hardkmeans}.

\newpage
\section{Expected generalization gap: Strongly convex and smooth case}\label{sec:expbound}
In this section, we bound the expected generalization gap using our localized cover for a contractive iterative stochastic optimizer $g$ (see Definition~\ref{def:optimizer}), i.e. each $g(\,\cdot\,;z)$  is contractive for all $z\in\mathcal Z$. 
Here, we use $\mathcal S$ for denoting the set of samples $\{z_1,\dots,z_n\}$ and $g_i$ for denoting $g(\,\cdot\,;z_i)$.
\begin{theorem}\label{thm:contractiveexp1}
Suppose that Assumptions \ref{asm:lipschitz} \& \ref{asm:boundedval} hold and there exists $\gamma\in(0,1)$ such that $g(\,\cdot\,;z)$ is $\gamma$-contractive on $\Theta\subseteq\mathcal B_R^d(0)$ for all $z\in\mathcal Z$.
Let 
$T:=\max\left\{\left\lceil\frac{\log(2LRn)}{\log(1/\gamma)}\right\rceil,0\right\}$. Then for any $\theta^{(0)}\in\Theta$, $t\ge0$, and $i_1,\dots,i_t\in[n]$, the $t$-th iterate $\theta^{(t)}:=g_{i_t}\circ\cdots\circ g_{i_1}(\theta^{(0)})$ satisfies 
\begin{align*}
|\mathbb E_{\mathcal S}[\hat F(\theta^{(t)})-F(\theta^{(t)})]|&\le\frac{BT+1}n.
\end{align*}
\end{theorem}
Compared to the uniform stability-based bound for $\gamma$-contractive $\frac1n\sum_{i=1}^ng_i$ under $L^\prime$-Lipschitz $f(\,\cdot\,;z)$ and $\|g(\theta;z)\|\le K$ \citep{kozachkov22}
\begin{align}
|\mathbb E_{\theta^{(0)},i_1,\dots,i_t,\mathcal S}[\hat F(\theta^{(t)})-F(\theta^{(t)})]|\le\frac{2KL^\prime}{(1-\gamma) n},\label{eq:stabilitybound}
\end{align}
our bound in Theorem~\ref{thm:contractiveexp1} is linear in $B$, has an additional $\log n$ factor, and requires contractive $g(\,\cdot\,;z)$ but does not depend on $L^\dagger$ and $K$. 
Here, note that our bound uniformly holds for all parameters generated by SGD while the bound \eqref{eq:stabilitybound} requires the expectation over $\theta^{(0)}$ and the indices $i_1,\dots,i_t$ used for the $t$ updates using the iterative stochastic optimizer. 

Although the bounds in Theorem~\ref{thm:contractiveexp1} and \eqref{eq:stabilitybound} provide some information about the expected generalization gap, these bounds cannot be used for understanding the absolute deviation $|\hat F(\theta^{(t)})-F(\theta^{(t)})|$, which can be of practical interest. 
In the following theorem, we provide a bound on the expectation of the absolute generalization gap where $\Psi_{\ge T}(\theta)$ and $\Psi_{T}(\theta)$ are defined for general $g_i$, analogous to Definition~\ref{def:cover}:
\begin{align}
    \Psi_T(\theta^{(0)}) :=&
    \{g_{i_T}\circ\cdots\circ g_{i_1}(\theta^{(0)}): i_1,\dots,i_T\in[n]\},\notag\\
    \Psi_{\geq T}(\theta^{(0)}) :=& \bigcup_{t\geq T}\Psi_t(\theta^{(0)}).\label{eq:cover}
\end{align}

\begin{theorem}\label{thm:contractiveexp2}
Assume the setup in Theorem~\ref{thm:contractiveexp1}. Then there exists some absolute constant $C>0$ such that
\begin{align*}
    \mathbb E_{\mathcal S}\left[\sup_{\theta^{(t)}\in\bigcup_{\theta^{(0)}\in\Theta}\Psi_{\ge T}(\theta^{(0)})}|\hat F(\theta^{(t)})-F(\theta^{(t)})|\right]\le\frac{BT+1}n+CB\sqrt{\frac{T\log n}{n}}.
\end{align*}
\end{theorem}
To our knowledge, the bound in Theorem~\ref{thm:contractiveexp2} is the first bound on $\mathbb E_{\mathcal S}[|\hat F(\theta^{(t)})-F(\theta^{(t)})|]$ for the contractive case, including strongly convex and smooth functions.
Compared to Theorem~\ref{thm:contractiveexp1}, the bound in Theorem~\ref{thm:contractiveexp2} looses $O(\sqrt{T/(n\log n)})$ factor for taking the absolute value inside the expectation. Nevertheless, this bound can be improved by removing the supremum inside the expectation.
\begin{corollary}\label{cor:contractiveexp3}
Assume the setup in Theorem~\ref{thm:strconvsmooth}. Then there exists some absolute constant $C>0$ such that for any $\theta^{(0)}\in\Theta$, $t\ge0$, and $i_1,\dots,i_t\in[n]$, we have
\begin{align*}
    \mathbb E_{\mathcal S}\left[|\hat F(\theta^{(t)})-F(\theta^{(t)})|\right]\le\frac{BT+1}n+CB\sqrt{\frac1{n}}.
\end{align*}
\end{corollary}
Compared to Theorem~\ref{thm:contractiveexp2}, the bound in Corollary~\ref{cor:contractiveexp3} does not have $\sqrt{T\log n}$ factor but the LHS in the bound is weaker.

\subsection{Proof of Theorem~\ref{thm:contractiveexp1}}\label{sec:pfthm:contractiveexp1}
For any $\theta^{(t)}\in\bigcup_{\theta^{(0)}\in\Theta}\Psi_{\ge0}(\theta^{(0)})$,
let $\theta^{(t)}=g_{i_t}\circ\cdots\circ g_{i_1}(\theta^{(0)})$. Let $\phi=g_{i_t}\circ\cdots\circ g_{i_{t-T+1}}(0)$ and $\mathcal I=\{i_{t-T+1},\dots,i_t\}\subset[n]$ if $t>T$ and $\phi=g_{i_t}\circ\cdots\circ g_{i_1}(0)$ and $\mathcal I=\{i_1,\dots,i_t\}\subset[n]$ otherwise. 
Then as in the proof of Lemma~\ref{lem:cover}, we have
\begin{align}
\theta^{(t)}\in\mathcal B_\varepsilon^d(\phi)\label{eq:pfthm:strconvsmoothexp1}    
\end{align}
for $\varepsilon=1/(2Ln)$, regardless of the choice of $\mathcal S$. Using this, we have
\begin{align*}
    &|\mathbb E_{\mathcal S}[\hat F(\theta^{(t)})-F(\theta^{(t)})]|\le|\mathbb E_{\mathcal S}[(\hat F(\theta^{(t)})-\hat F(\phi)]|+|\mathbb E_{\mathcal S}[F(\theta^{(t)})-F(\phi)]|+|\mathbb E_{\mathcal S}[\hat F(\phi)-F(\phi)]|\\
    &\le2L\varepsilon+|\mathbb E_{\mathcal S}[\hat F(\phi)-F(\phi)]|\\
    &\le2L\varepsilon+\bigg|\mathbb E_{\mathcal S}\bigg[\frac{1}n\sum_{i\in[n]\setminus\mathcal I}(\hat f(\phi;z_i)-F(\phi))\bigg]\bigg|+\bigg|\mathbb E_{z_{i}:i\in\mathcal I}\bigg[\frac{1}n\sum_{i\in\mathcal I}(\hat f(\phi;z_i)-F(\phi))\bigg]\bigg|\\
    &=2L\varepsilon+\bigg|\mathbb E_{z_{i}:i\in\mathcal I}\bigg[\frac{1}n\sum_{i\in\mathcal I}(\hat f(\phi;z_i)-F(\phi))\bigg]\bigg|\\
    &\le2L\varepsilon+\frac{BT}n=\frac{BT+1}n.
\end{align*}
The first inequality is from the triangle inequality and the second inequality is from \eqref{eq:pfthm:strconvsmoothexp1}. The third inequality is again from the triangle inequality and the first equality is from the fact that $\phi$ is independent of $\{z_i:i\in\mathcal I\}$.
The last inequality is from Assumption~\ref{asm:boundedval}.
Since the above bound holds for any $\theta^{(0)}\in\Theta$ and $\mathcal I_t$, this completes the proof of Theorem~\ref{thm:contractiveexp1}.

\subsection{Proof of Theorem~\ref{thm:contractiveexp2}}\label{sec:pfthm:contractiveexp2}
In this proof, we assume $T<n$ since the statement trivially follows otherwise. As in the statement of Lemma~\ref{lem:cover}, one can observe that 
\begin{align}
    \bigcup_{\theta^{(0)}\in\Theta}\Psi_{\geq T}(\theta^{(0)}) \subseteq \bigcup_{\phi \in \Psi_T(0)}\mathcal B_\varepsilon^d(\phi)\label{eq:pfthm:strconvsmoothexp2}
\end{align}
for $\varepsilon=1/(2Ln)$.
For each $\phi\in\Psi_T(0)$, let $\mathcal I_\phi:=\{i_1,\dots,i_T\}$ such that $\phi=g_{i_T}\circ\cdots\circ g_{i_1}(0)$.
Using this we can derive the following bound: for $\Pi_{\Psi_T(0)}(\theta^{(t)}):=\arg\min_{\phi\in\Psi_T(0)}\|\theta^{(t)}-\phi\|$, 
\begin{align*}
    &\mathbb E_{\mathcal S}\left[\sup_{\theta^{(t)}\in\bigcup_{\theta^{(0)}\in\Theta}\Psi_{\ge T}(\theta^{(0)})}|\hat F(\theta^{(t)})-F(\theta^{(t)})|\right]\le\mathbb E_{\mathcal S}\left[\sup_{\phi\in\Psi_T(0)}|\hat F(\phi)-F(\phi)|\right]\\
    &~~+\mathbb E_{\mathcal S}\left[\sup_{\theta^{(t)}\in\bigcup_{\theta^{(0)}\in\Theta}\Psi_{\ge T}(\theta^{(0)})}|\hat F(\theta^{(t)})-\hat F(\Pi_{\Psi_T(0)}(\theta^{(t)}))|+|F(\theta^{(t)})-F(\Pi_{\Psi_T(0)}(\theta^{(t)}))|\right]\\
    &\le 2L\varepsilon+\mathbb E_{\mathcal S}\left[\sup_{\phi\in\Psi_T(0)}|\hat F(\phi)-F(\phi)|\right]\\
    &\le\frac1n+\mathbb E_{\mathcal S}\left[\sup_{\phi\in\Psi_T(0)}\left|\frac{1}n\sum_{i\in\mathcal I_\phi}(f(\phi;z_i)-F(\phi))\right|\right]+\mathbb E_{\mathcal S}\left[\sup_{\phi\in\Psi_T(0)}\left|\frac{1}n\sum_{i\in[n]\setminus\mathcal I_\phi}(f(\phi;z_i)-F(\phi))\right|\right]\\
    &\le\frac{BT+1}n+\mathbb E_{\mathcal S}\left[\sup_{\phi\in\Psi_T(0)}\left|\frac{|[n]\setminus\mathcal I_\phi|}n\cdot\frac1{|[n]\setminus\mathcal I_\phi|}\sum_{i\in[n]\setminus\mathcal I_\phi}(f(\phi;z_i)-F(\phi))\right|\right]\\
    &\le\frac{BT+1}n+CB\sqrt{\frac{T\log n}{n}}.
\end{align*}
The first inequality is from the triangle inequality and the second inequality is from \eqref{eq:pfthm:strconvsmoothexp2}. 
The third inequality is again from the triangle inequality and the definition of $\varepsilon$, and the fourth inequality is from Assumption~\ref{asm:boundedval}. Since 
each $\frac1{|[n]\setminus\mathcal I_\phi|}\sum_{i\in[n]\setminus\mathcal I_\phi}(f(\phi;z_i)-F(\phi))$ is sub-Gaussian with the sub-Gaussian norm bounded by $cB/{\sqrt{|[n]\setminus\mathcal I_\phi|}}$ for some absolute constant $c$ (see Proposition 2.5.2 in \cite{vershynin18}),
the last inequality follows from a standard upper bound for empirical processes (see Exercise 2.5.10 in \cite{vershynin18}). 
This completes the proof of Theorem~\ref{thm:contractiveexp2}.

\subsection{Proof of Corollary~\ref{cor:contractiveexp3}}\label{sec:pfcor:contractiveexp3}
As in the proof of Theorem~\ref{thm:contractiveexp2}, we assume that $T<n$ without loss of generality. The proof here is almost identical to that of Theorems~\ref{thm:contractiveexp1} \& \ref{thm:contractiveexp2}.

For any $\theta^{(t)}\in\bigcup_{\theta^{(0)}\in\Theta}\Psi_{\ge0}(\theta^{(0)})$,
let $\theta^{(t)}=g_{i_t}\circ\cdots\circ g_{i_1}(\theta^{(0)})$. Let $\phi=g_{i_t}\circ\cdots\circ g_{i_{t-T+1}}(0)$ and $\mathcal I=\{i_{t-T+1},\dots,i_t\}\subset[n]$ if $t>T$ and $\phi=g_{i_t}\circ\cdots\circ g_{i_1}(0)$ and $\mathcal I=\{i_1,\dots,i_t\}\subset[n]$ otherwise. Then for $\varepsilon=1/(2Ln)$, we have
\begin{align}
    \theta^{(t)}\in\mathcal B_\varepsilon(\phi).\label{eq:pfcor:strconvsmoothexp3} 
\end{align}
Using this, we derive the following inequality:
\begin{align*}
    &\mathbb E_{\mathcal S}[|\hat F(\theta^{(t)})-F(\theta^{(t)})|]\le\mathbb E_{\mathcal S}[|(\hat F(\theta^{(t)})-\hat F(\phi)|]+\mathbb E_{\mathcal S}[|F(\theta^{(t)})-F(\phi)|]+\mathbb E_{\mathcal S}[|\hat F(\phi)-F(\phi)|]\\
    &\le2L\varepsilon+\mathbb E_{\mathcal S}[|\hat F(\phi)-F(\phi)|]\\
    &\le2L\varepsilon+\mathbb E_{\mathcal S}\bigg[\bigg|\frac{1}n\sum_{i\in[n]\setminus\mathcal I}(\hat f(\phi;z_i)-F(\phi))\bigg|\bigg]+\mathbb E_{z_{i}:i\in\mathcal I}\bigg[\bigg|\frac{1}n\sum_{i\in\mathcal I}(\hat f(\phi;z_i)-F(\phi))\bigg|\bigg]\\
    &\le2L\varepsilon+\mathbb E_{\mathcal S}\bigg[\bigg|\frac{1}n\sum_{i\in[n]\setminus\mathcal I}(\hat f(\phi;z_i)-F(\phi))\bigg|\bigg]+\frac{BT}n\\
    &\le\frac{BT+1}n+\frac{|[n]\setminus\mathcal I|}{n}\mathbb E_{z_{i}:i\in[n]\setminus\mathcal I}\bigg[\bigg|\frac1{|[n]\setminus\mathcal I|}\sum_{i\in[n]\setminus\mathcal I}(\hat f(\phi;z_i)-F(\phi))\bigg|\bigg]\\
    &\le\frac{BT+1}n+\frac{|[n]\setminus\mathcal I|}n\left(CB\sqrt{\frac{1}{|[n]\setminus\mathcal I|}}\right)\\
    &\le\frac{BT+1}n+CB\sqrt{\frac{1}{n}}.
\end{align*}
The first and second inequality follows from the triangle inequality and \eqref{eq:pfcor:strconvsmoothexp3}, respectively. 
The third inequality is again from the triangle inequality while Assumption~\ref{asm:boundedval} gives us the fourth inequality. The sixth inequality is from Proposition 2.5.2 and Exercise 2.5.10 in \cite{vershynin18} where the last inequality naturally follows. Since the above bound holds for any $\theta^{(0)}\in\Theta$ and $\mathcal I_t$, this completes the proof of Corollary~\ref{cor:contractiveexp3}

\newpage
\section{Comparison with existing bounds}\label{sec:comparison}
\renewcommand{\arraystretch}{1.4}
\begin{table}[H]
\centering
{\caption{Summary of generalization bounds for constant step-size SGD. In the column ``Assumptions'', Lip.\ assumes $L^\prime$-Lipschitz $f(\,\cdot\,;z)$, Weak Lip.\ assumes Assumption~\ref{asm:lipschitz}, Bdd.\ assumes $f(\Theta;\mathcal Z)\subset[0,B^\prime]$, and Bdd.\ Dev.\ assumes Assumption~\ref{asm:boundedval}.
If LHS of a bound does not contain expectation, that bound is a high-probability bound. 
For simplicity, we hide values other than $d,t,T,L,L^\prime,B,B^\prime,P,\xi,n$ in~$\lesssim$ where $T=\log(LRn)$ and $\Theta\subset\mathcal B_R^d(0)$. }\label{table:comparison}
\footnotesize
\begin{tabular}{|c|c|c|c|c|}
\hline
\multirow{2}{*}{Reference} & \multirow{2}{*}{Objective} & \multirow{2}{*}{Assumptions} & \multirow{2}{*}{\!\!\makecell{Stable\\as $t\!\to\!\infty$}\!\!} & \multirow{2}{*}{Bound: $\Delta^{\!(t)}\!=\!\hat F(\theta^{(t)})\!-\!F(\theta^{(t)})$}\\
& & & &\\\hline
\multirow{1}{*}{\makecell{\cite{hardt15}}}
& \multirow{6}{*}{\makecell{Strongly convex\\
\& smooth}} & \multirow{1}{*}{Lip.}& \multirow{6}{*}{\large\cmark} & \multirow{1}{*}{$|\mathbb E[\Delta^{\!(t)}]|\lesssim{(L^\prime)^2}/{n}\,^*$}\\
\cline{1-1}\cline{3-3}\cline{5-5}
Thm~\ref{thm:contractiveexp1}
& & \multirow{5}{*}{\makecell{Weak Lip.,\\\& Bdd.\ Dev.}}& &$|\mathbb E[\Delta^{\!(t)}]|\lesssim {BT}/n\,^{\|}$\\
\cline{1-1}\cline{5-5}
Thm~\ref{thm:contractiveexp2} &  & & &$\mathbb E[\sup|\Delta^{\!(t)}|]\lesssim B\sqrt{{T\log n}/{n}}\,^{\ddagger\|}$\\\cline{1-1}\cline{5-5}
Cor~\ref{cor:contractiveexp3} & &  & & $\mathbb E[|\Delta^{\!(t)}|]\lesssim B\sqrt{1/{n}}\,^{\|}$\\\cline{1-1}\cline{5-5}
Thm~\ref{thm:strconvsmooth} & & & & $\sup|\Delta^{\!(t)}|\lesssim B\sqrt{{T\log n}/{n}}\,^\ddagger$\\ \cline{1-1}\cline{5-5}
Cor~\ref{cor:strconvsmooth}--\ref{thm:early} &  & & & $|\Delta^{\!(t)}|\lesssim B\sqrt{{1}/{n}}$\\
\hline
\cite{hardt15} & \multirow{2}{*}{\makecell{Convex\\\& smooth}}& Lip. & \multirow{5}{*}{\large\xmark} & $|\mathbb E[\Delta^{\!(t)}]|\lesssim(L^\prime)^2t/n\,^*$\\\cline{1-1}\cline{3-3}\cline{5-5}
\cite{feldman19} &  & Lip.\ \& Bdd.& & $|\Delta^{\!(t)}|\lesssim t\log^2n/n+\sqrt{1/n}\,^\mathparagraph$\\\cline{1-3}\cline{5-5}
\multirow{2}{*}{\cite{bassily20}} & \multirow{2}{*}{\makecell{Convex\\\& non-smooth}} & Lip. & &$|\mathbb E[\Delta^{\!(t)}]|\lesssim(L^\prime)^2(\sqrt{t}+t/n)\,^*$\\\cline{3-3}\cline{5-5}
 &  & Lip.\ \&  Bdd.& & $|\Delta^{\!(t)}|\lesssim (L^\prime\log n)^2(\sqrt{t}+t/n)+B^\prime\sqrt{1/n}$\\\cline{1-3}\cline{5-5}
\cite{hardt15} & \multirow{2}{*}{\makecell{Non-convex\\\& smooth}} & Lip.& &$|\mathbb E[\Delta^{\!(t)}]|\lesssim{(L^\prime)^{\frac2{\beta c+1}}t^{\frac{\beta c}{\beta c+1}}}/{n}\,^{*\mathsection}$\\
\cline{1-1}\cline{3-5}
 \multirow{4}{*}{Thm~\ref{thm:approxpiecestrconvsmooth}}&  &\multirow{4}{*}{\makecell{Weak Lip.\\\& Bdd.\ Dev.}} & \multirow{4}{*}{\large\cmark}& $\sup|\Delta^{\!(t)}|\lesssim B\sqrt{{dT^2\log n}/{n}}\,^\ddagger$\\
\cline{2-2}\cline{5-5}
 & \multirow{3}{*}{\makecell{Approx.\ piecewise\\ strongly convex\\\& smooth}} & & & \multirow{3}{*}{$\sup|\Delta^{\!(t)}|\lesssim B\sqrt{{T\log (nP)}/{n}}+L\xi\,^\ddagger$}\\
 & & & &\\
 & & & &\\
\hline
\end{tabular}
\vspace{0.05in}
\caption{Summary of generalization bounds for a (piecewise) contractive optimizer with update functions $g_1,\dots,g_n$. The column ``Non-cvx SGD'' evaluates if a bound can be used for SGD on non-convex objectives without diverging as $t\rightarrow\infty$.}\label{table:comparison2}
\begin{tabular}{|c|c|c|c|c|}
\hline
\multirow{2}{*}{Reference} & \multirow{2}{*}{Optimizer} & \multirow{2}{*}{Assumptions} & \multirow{2}{*}{\!\!\makecell{Non-cvx\\SGD}\!\!} & \multirow{2}{*}{Bound: $\Delta^{\!(t)}\!=\!\hat F(\theta^{(t)})\!-\!F(\theta^{(t)})$}\\
& & & &\\\hline
\multirow{2}{*}{\cite{kozachkov22}}
& \multirow{2}{*}{\makecell{Contractive\\ $\frac1n\sum_{i=1}^ng_i$}} & \multirow{2}{*}{\makecell{Lip.\ \& \\$\|g_i(\Theta)\|\le K$}}& \multirow{2}{*}{\large\xmark}& \multirow{2}{*}{$|\mathbb E[\Delta^{\!(t)}]|\lesssim{KL^\prime}/{n}\,^{*}$}\\
& & & & \\
\hline
Thm~\ref{thm:contractiveexp1}
& \multirow{5}{*}{\makecell{Contractive $g_i$}} & \multirow{5}{*}{\makecell{Weak Lip.\\\& Bdd.\ Dev.}}& \multirow{5}{*}{\large\xmark}& $|\mathbb E[\Delta^{\!(t)}]|\lesssim {BT}/n\,^{\|}$\\
\cline{1-1}\cline{5-5}
Thm~\ref{thm:contractiveexp2} & &  & &$\mathbb E[\sup|\Delta^{\!(t)}|]\lesssim B\sqrt{{T\log n}/{n}}\,^{\ddagger\|}$\\\cline{1-1}\cline{5-5}
Cor~\ref{cor:contractiveexp3} & & & & $\mathbb E[|\Delta^{\!(t)}|]\lesssim B\sqrt{1/{n}}\,^{\|}$\\\cline{1-1}\cline{5-5}
Thm~\ref{thm:piececontractive} & & & & $\sup|\Delta^{\!(t)}|\lesssim B\sqrt{{T\log n}/{n}}\,^\ddagger$\\ \cline{1-1}\cline{5-5}
Cor~\ref{cor:strconvsmooth}--\ref{thm:early}\,$^\flat$ & & & & $|\Delta^{\!(t)}|\lesssim B\sqrt{{1}/{n}}$\\
\hline
\multirow{2}{*}{Thm~\ref{thm:piececontractive}} & \multirow{2}{*}{\makecell{Approx.\ piecewise\\contractive $g_i$}} & \multirow{2}{*}{\makecell{Weak Lip.\\\& Bdd.\ Dev.}}& \multirow{2}{*}{\large\cmark} & \multirow{2}{*}{$\sup|\Delta^{\!(t)}|\lesssim B\sqrt{{T\log (nP)}/{n}}+L\xi\,^\ddagger$}\\
& & & &\\
\hline
\end{tabular}
{\footnotesize
\begin{tablenotes}
\item[] $^*$The expectation is taken over $\{z_1,\dots,z_n\},\{i_1,\dots,i_t\},\theta^{(0)}$.
\item[] $^\|$The expectation is taken over $\{z_1,\dots,z_n\}$.
\item[] $^\ddagger$The supremum is taken over $\theta^{(t)}\in\bigcup_{\theta^{(0)}\in\Theta}\Psi_{\ge T}(\theta^{(0)})$ (see Definition~\ref{def:cover} and \eqref{eq:cover}).
\item[] $^\mathsection\beta$ denotes the smoothness parameter and the adaptive learning rate must satisfy $\eta_t\le c/t$.
\item[] $^\mathparagraph$This bound is under $1$-Lipschitz continuity of $f(\,\cdot\,;z)$ and $f(\Theta,\mathcal Z)\subset[0,1]$.
\item[] $^\flat$Although Corollaries~\ref{cor:strconvsmooth}--\ref{thm:early} are for SGD, the same result holds for contractive $g_i$ with the same proof.
\end{tablenotes}
}}
\vspace{-0.1in}
\end{table}
\renewcommand{\arraystretch}{1}

\newpage
\section{Lower bound on uniform stability-based bound for piecewise strongly convex and smooth objectives}\label{sec:stabilityfail}

In this section, we show that the uniform stability-based bound \cite{hardt15} is $\Omega(1)$ for piecewise strongly convex and smooth functions after sufficiently many SGD iterations in general.
To this end, we first introduce the formal definition of the uniform stability and a standard tool for showing that the stability implies generalization, which are from Definition 2.1 and Theorem 2.2 in \cite{hardt15}.
\begin{mydefinition}[Uniform stability]
A randomized algorithm $\rho$ is ``$\varepsilon$-uniformly stable'' if for all data sets $\mathcal S,\mathcal S^\prime\in\mathcal Z^n$ such that $\mathcal S$ and $\mathcal S^\prime$ differ in at most one example, we have
\begin{align*}
    \sup_{z\in\mathcal Z}\mathbb E_\rho[f(\rho(\mathcal S);z)-f(\rho(\mathcal S^\prime);z)]\le\varepsilon.
\end{align*}
\end{mydefinition}

\begin{theorem}\label{thm:stability} If $\rho$ is $\varepsilon$-uniformly stable, then
\begin{align*}
    |\mathbb E_{\mathcal S,\rho}[\hat F(\rho(\mathcal S))-F(\rho(\mathcal S))]|\le\varepsilon.
\end{align*}
\end{theorem}
For the remaining section, we provide an example that
\begin{align*}
\sup_{z\in\mathcal Z}\mathbb E_\rho[f(\rho(\mathcal S);z)-f(\rho(\mathcal S^\prime);z)]=\Omega(1),
\end{align*}
regardless of $n$ where $\rho$ denotes sufficiently many SGD updates. Namely, the uniform stability-based bound based on Theorem~\ref{thm:stability} is $\Omega(1)$.

Let $\Theta=[0,4]$, $\mathcal Z=\{0,1\}$, $\theta^{(0)}\sim\text{Unif}([0,4])$, $\mathbb P(z=0)=\mathbb P(z=1)=\frac12$, $f(\,\cdot\,;0)=\min\{(x-1)^2,\frac12+\frac12(x-3)^2\}$, $f(\,\cdot\,;1)=(x-1)^2$, and the auxiliary gradient $\nabla f(2;0)=2$ at the non-differentiable point $\theta=2$, i.e., $f(\,\cdot\,;z)$ is piecewise $1$-strongly convex and $2$-smooth with the partition $\mathcal P=\{[0,2],(2,4]\}$. 
For simplicity, choose $\eta=1/3$ which can be generalized to arbitrary $\eta\in(0,1)$.
Under this setup, using Theorem~\ref{thm:approxpiecestrconvsmooth}, one can easily derive a generalization bound that does not increase with the number of SGD iterations and converges to zero as $n$ grows.

However, under the same setup and sufficiently many SGD iterations, the uniform stability-based bound is lower bounded by a constant regardless of $n$. To see this, let $\mathcal S=(0,\dots,0)$ and $\mathcal S^\prime=(1,0,\cdots,0)$. Then, one can observe that $f([0,2];0)\subset[0,2]$ and $f((2,4];0)\subset(2,4]$, i.e. SGD iterates for $\mathcal S$ converge to either $1$ or $3$ depending on whether $\theta^{(0)}\in[0,2]$ or $\theta^{(0)}\in(2,4]$.
Since we assumed $\theta^{(0)}\sim\text{Unif}([0,4])$, we have 
\begin{align}
    \lim_{t\rightarrow\infty}\mathbb E_{\theta^{(t)}}[f(\theta^{(t)};1)]=\frac12 f(1;1)+\frac12 f(3;1)=2\label{eq:unifstabex1}
\end{align}
where $\theta^{(t)}$ denotes a random parameter generated by $t$ SGD updates for $\mathcal S$, from $\theta^{(0)}\sim\text{Unif}([0,4])$.
Furthermore, we have $f(\Theta;1)\subset[0,2]$ and $f([0,2];0)\subset[0,2]$. This implies that if a single SGD update for $\mathcal S^\prime$ use the first sample in $\mathcal S^\prime$ (i.e. $z=1$), which occurs with high probability under sufficiently many SGD iterations, then the SGD iterates will converge to $1$ almost surely.
In other words, we have
\begin{align}
    \lim_{t\rightarrow\infty}\mathbb E_{\phi^{(t)}}[f(\phi^{(t)};1)]= f(1;1)=0\label{eq:unifstabex2}
\end{align}
where $\phi^{(t)}$ denotes a random parameter generated by $t$ SGD updates for $\mathcal S^\prime$, from $\phi^{(0)}\sim\text{Unif}([0,4])$.
Combining \eqref{eq:unifstabex1} and \eqref{eq:unifstabex2} implies a constant lower bound on the uniform stability-based bound, regardless of $n$, under sufficiently many SGD iterations.
We note that the same conclusion can also be derived for any $\eta\in(0,1)$ as long as the number of SGD iterations is large enough.

\end{document}